\documentclass{article}

\usepackage[nonatbib, final]{neurips_2020_mod}

\usepackage[utf8]{inputenc} %
\usepackage[T1]{fontenc}    %
\usepackage{hyperref}       %
\usepackage{url}            %
\usepackage{booktabs}       %
\usepackage{amsfonts}       %
\usepackage{nicefrac}       %
\usepackage{microtype}      %
\usepackage{adjustbox}

\usepackage{hyperref}

\usepackage{amsmath,amssymb,amsthm}
\usepackage{algorithm,algorithmic}
\usepackage[dvipsnames]{xcolor}
\usepackage{mathtools}
\usepackage{graphicx}
\usepackage{hyperref}
\usepackage{float}
\usepackage[numbers]{natbib}
\usepackage{lipsum} %
\usepackage{bm}
\usepackage{makecell}
\usepackage{enumitem}
\usepackage{caption}
\usepackage{adjustbox}
\usepackage{subcaption}
\usepackage{wrapfig, blindtext}

\hypersetup{
 colorlinks=True,
 linkcolor=blue,
 citecolor=blue,
 urlcolor=blue}

\newtheorem{theorem}{Theorem}
\newtheorem{lemma}[theorem]{Lemma}
\newtheorem{corollary}[theorem]{Corollary}

\newtheorem{proposition}[theorem]{Proposition}
\newtheorem{definition}{Definition}

\newcommand{\Rmax}{R_{\textrm{max}}}

\newcommand{\E}[1]{\mathbb{E}\left[#1\right]}

\newcommand{\norm}[1]{\left\| #1 \right\|}
\newcommand{\defeq}{\stackrel{\textrm{def}}{=}}
\newcommand{\abs}[1]{\left| #1 \right|}
\newcommand{\set}[1]{\left\{#1\right\}}

\newcommand{\M}{\mathcal{M}}
\newcommand{\Mp}{\hat{\mathcal{M}}_p}
\newcommand{\Mpe}{{\mathcal{M}}_p}
\newcommand{\fname}{\texttt{MOReL}}

\newcommand{\U}{\mathcal{U}}

\newcommand{\data}{\mathcal{D}}
\newcommand{\nrew}{\kappa}
\newcommand{\walker}{\texttt{Walker2d-v2}}
\newcommand{\hopper}{\texttt{Hopper-v2}}
\newcommand{\hc}{\texttt{HalfCheetah-v2}}
\newcommand{\ant}{\texttt{Ant-v2}}

\title{MOReL: Model-Based Offline \\ Reinforcement Learning}

\author{%
Rahul Kidambi$^*$ \\
Cornell University, Ithaca \\
\texttt{rkidambi@cornell.edu} \\
\And
Aravind Rajeswaran\thanks{Equal Contributions. Correspond to \texttt{rkidambi@cornell.edu} \ and \ \texttt{aravraj@cs.washington.edu}. This version of the paper extends the results presented at NeurIPS 2020 through the addition of results in the D4RL benchmark suite and by expanding the scope of Lemma~\ref{lem:asymptotic}.
} \\
University of Washington, Seattle \\ Google Brain, Mountain View \\
\texttt{aravraj@cs.washington.edu} \\
\AND
\hspace*{-22pt} Praneeth Netrapalli \\
\hspace*{-22pt} Microsoft Research, India \\
\hspace*{-22pt} \texttt{praneeth@microsoft.com} \\
\And
\hspace*{-10pt} Thorsten Joachims \\
\hspace*{-10pt} Cornell University, Ithaca \\
\hspace*{-10pt} \texttt{tj@cs.cornell.edu} \\
}

\begin{document}

\maketitle

\begin{abstract}
In offline reinforcement learning (RL), the goal is to learn a highly rewarding policy based solely on a dataset of historical interactions with the environment. 
The ability to train RL policies offline would greatly expand where RL can be applied, its data efficiency, and its experimental velocity.
Prior work in offline RL has been confined almost exclusively to model-free RL approaches. In this work, we present~\fname, an algorithmic framework for model-based offline RL. This framework consists of two steps: (a)~learning a {\em pessimistic MDP}~(P-MDP) using the offline dataset; (b)~learning a near-optimal policy in this P-MDP. The learned P-MDP has the property that for {\em any} policy, the performance in the real environment is approximately lower-bounded by the performance in the P-MDP. This enables it to serve as a good surrogate for purposes of policy evaluation and learning, and overcome common pitfalls of model-based RL like model exploitation. Theoretically, we show that~\fname~is minimax optimal (up to log factors) for offline RL. Through experiments, we show that~\fname~matches or exceeds state-of-the-art results in widely studied offline RL benchmarks. Moreover, the modular design of~\fname~enables future advances in its components (e.g., in model learning, planning etc.) to directly translate into improvements for offline RL.

\end{abstract}

\section{Introduction}\label{sec:intro}
The fields of computer vision and NLP have seen tremendous advances by utilizing large-scale offline datasets~\cite{bb77703,FisherDoddington86,ChelbaMSGBK13} for training and deploying deep learning models~\citep{KrizhevskySH17,HintonETAL12,MikolovSCCD13,PetersETAL18}. In contrast, reinforcement learning (RL)~\cite{suttonReinforcement1998} is typically viewed as an online learning process. The RL agent iteratively collects data through interactions with the environment while learning the policy.
Unfortunately, a direct embodiment of this trial and error learning is often inefficient and feasible only with a simulator~\cite{OpenAIHand, AlphaZero, AlphaStar}. Similar to progress in other fields of AI, the ability to learn from offline datasets may hold the key to unlocking the sample efficiency and widespread use of RL agents.

Offline RL, also known as batch RL~\cite{LangeGR12}, involves learning a highly rewarding policy using only a static offline dataset collected by one or more data logging (behavior) policies. Since the data has already been collected, offline RL abstracts away data collection or exploration, and allows prime focus on data-driven learning of policies. This abstraction is suitable for safety sensitive applications like healthcare and industrial automation where careful oversight by a domain expert is necessary for taking exploratory actions or deploying new policies~\cite{Thomas14,thomas2019preventing}. Additionally, large historical datasets are readily available in domains like autonomous driving and recommendation systems, where offline RL may be used to improve upon currently deployed policies.

Due to use of static dataset, offline RL faces unique challenges. Over the course of learning, the agent has to evaluate and reason about various candidate policy updates. This {\em offline policy evaluation} is particularly challenging due to deviation between the state visitation distribution of the candidate policy and the logging policy. Furthermore, this difficulty is exacerbated over the course of learning as the candidate policies increasingly deviate from the logging policy. This change in distribution, as a result of policy updates, is typically called {\em distribution shift} and constitutes a major challenge in offline RL. Recent studies show that directly using off-policy RL algorithms with an offline dataset yields poor results due to distribution shift and function approximation errors~\cite{FujimotoBCQ, KumarBEAR, AgarwalSN19}. To overcome this, prior works have proposed modifications like Q-network ensembles~\cite{FujimotoBCQ,WuTN19BRAC} and regularization towards the data logging policy~\cite{JaquesGHSFLJGP19, KumarBEAR, WuTN19BRAC}. Most notably, prior work in offline RL has been confined almost exclusively to model-free methods~\cite{LarocheT17,FujimotoBCQ,KumarBEAR, JaquesGHSFLJGP19,AgarwalSN19,WuTN19BRAC,Nachum19AlgaeDICE}. 

Model-based RL (MBRL) presents an alternate set of approaches involving the learning of approximate dynamics models which can subsequently be used for policy search. MBRL enables the use of generic priors like smoothness and physics~\cite{Zeng2019TossingBotLT} for model learning, and a wide variety of planning algorithms~\cite{Todorov2005, tassa2012synthesis, MCTSSurvey, Munos2008FiniteTimeBF, SchulmanPPO}. As a result, MBRL algorithms have been highly sample efficient for online RL~\cite{RajeswaranGameMBRL, JannerFZL19}. However, direct use of MBRL algorithms with offline datasets can prove challenging, again due to the distribution shift issue. In particular, since the dataset may not span the entire state-action space, the learned model is unlikely to be globally accurate. As a result, planning using a learned model without any safeguards against model inaccuracy can result in {\em ``model exploitation''}~\cite{KurutachCDTA18, Clavera2018ModelBasedRL, JannerFZL19, RajeswaranGameMBRL}, yielding poor results~\cite{RossB12}. In this context, we study the pertinent question of how to effectively regularize and adapt model-based methods for offline RL.

\begin{figure}[t!]
    \centering
    \hspace*{-10pt}
    \vspace*{-5pt}
    \includegraphics[width=1.05\textwidth]{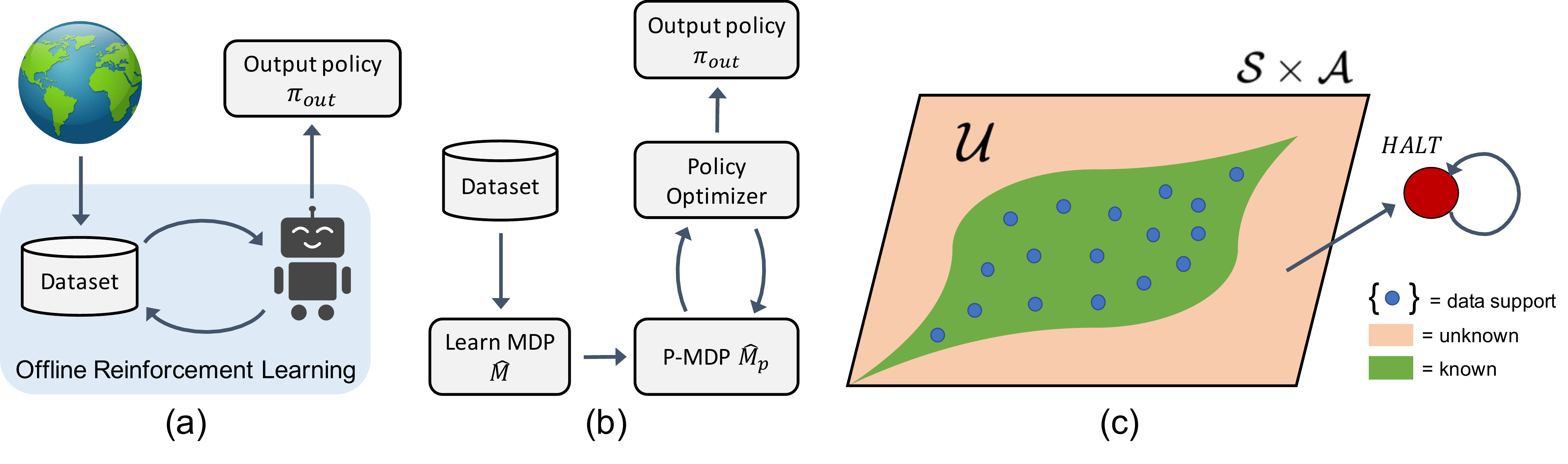}
    \caption{(a) Illustration of the offline RL paradigm. (b) Illustration of our framework,~\fname, which learns a pessimistic MDP (P-MDP) from the dataset and uses it for policy search. (c) Illustration of the P-MDP, which partitions the state-action space into known (green) and unknown (orange) regions, and also forces a transition to a low reward absorbing state (HALT) from unknown regions. Blue dots denote the support in the dataset. Function approximation and generalization allows us to learn about states not numerically identical to data support. See algorithm~\ref{alg:mbpm} for more details.}
    \vspace*{-10pt}
    \label{fig:illustration}
\end{figure}

{\bf Our Contributions:} The principal contribution of our work is the development of~\fname~(\underline{M}odel-based \underline{O}ffline \underline{Re}inforcement \underline{L}earning), a novel model-based framework for offline RL (see figure~\ref{fig:illustration} for an overview).~\fname~enjoys rigorous theoretical guarantees, enables transparent algorithm design, and offers state of the art (SOTA) results on widely studied offline RL benchmarks. 
\vspace*{-5pt}
\begin{itemize}[leftmargin=*]
    \itemsep0em
    \item \fname~consists of two modular steps: (a)~learning a {\em pessimistic MDP}~(P-MDP) using the offline dataset; and (b)~learning a near-optimal policy for the P-MDP. For {\em any} policy, the performance in the true MDP (environment) is approximately lower bounded by the performance in the P-MDP, making it a suitable surrogate for purposes of policy evaluation and learning. This also guards against model exploitation, which often plagues MBRL.
    \item The P-MDP partitions the state space into ``known'' and ``unknown'' regions, and uses a large negative reward for unknown regions. This provides a regularizing effect during policy learning by heavily penalizing policies that visit unknown states. Such a regularization in the space of state visitations, afforded by a model-based approach, is particularly well suited for offline RL. In contrast, model-free algorithms~\cite{KumarBEAR, WuTN19BRAC} are forced to regularize the policies directly towards the data logging policy, which can be overly conservative.
    \item Theoretically, we establish upper bounds for the sub-optimality of a policy learned with~\fname, and a lower-bound for the sub-optimality of a policy learnable by {\em any} offline RL algorithm. We find that these bounds match upto log factors, suggesting that~\fname~is nearly minimax optimal.
    \item We evaluate~\fname~on standard benchmark tasks used for offline RL.~\fname~obtains SOTA results in $12$ out of $20$ environment-dataset configurations, and performs competitively in the rest. In contrast, the best prior algorithm~\citep{WuTN19BRAC} obtains SOTA results in only $5$ (out of $20$) configurations.
\end{itemize}

In addition, this version of the paper extends the results presented at NeurIPS 2020 through addition of results in the D4RL benchmark suite and by expanding the scope of Lemma~\ref{lem:asymptotic}.

\section{Related Work}\label{sec:related}
Offline RL dates to at least the work of~\citet{LangeGR12}, and has applications in healthcare~\cite{Gottesman18,WangZHZ18,YuR019}, recommendation systems~\cite{StrehlLK10, SwaminathanJ15, CovingtonAS16, ChenBCJBC18}, dialogue systems~\cite{ZhouSRE17,JaquesGHSFLJGP19,Karampatziakis19}, and autonomous driving~\cite{SallabAPY17}. Algorithms for offline RL typically fall under three categories.
The first approach utilizes {\bf importance sampling} and is popular in contextual bandits~\cite{LiCLW10, StrehlLK10, SwaminathanJ15}. For full offline RL,~\citet{LiuSAB19} perform planning with learned importance weights~\cite{HallakM17, GeladaB19, NachumDualDICE} while using a notion of pessimism for regularization. However,~\citet{LiuSAB19} don't explicitly consider generalization and their guarantees become degenerate if the logging policy does not span the support of the optimal policy. In contrast, our approach accounts for generalization, leads to stronger theoretical guarantees, and obtains SOTA results on challenging offline RL benchmarks.
The second, and perhaps most popular approach is based on {\bf approximate dynamic programming~(ADP)}. Recent works have proposed modification to standard ADP algorithms~\cite{Watkins89, MnihKSRVBGRFOPB15, LillicrapHPHETS15, HaarnojaSAC} towards stabilizing Bellman targets with ensembles~\cite{AgarwalSN19, FujimotoBCQ, JaquesGHSFLJGP19} and regularizing the learned policy towards the data logging policy~\cite{FujimotoBCQ, KumarBEAR, WuTN19BRAC}. ADP-based offline RL has also be studied theoretically~\cite{Munos2008FiniteTimeBF, Chen2019InformationTheoreticCI}. 
However, these works again don't study the impact of support mismatch between logging policy and optimal policy.
Finally, {\bf model-based RL} has been explored only sparsely for offline RL in literature~\citep{RossB12,GhavamzadehPC16} (see appendix for details). The work of \citet{RossB12} considered a straightforward approach of learning a model from offline data, followed by planning. They showed that this can have arbitrarily large sub-optimality. In contrast, our work develops a new framework utilizing the notion of pessimism, and shows both theoretically and experimentally that MBRL can be highly effective for offline RL. Concurrent to our work, Yu et al.~\cite{MOPO} also study a model-based approach to offline RL.

A cornerstone of~\fname~is the P-MDP which partitions the state space into known and unknown regions. Such a hard partitioning was considered in early works like $E^3$~\cite{KearnsSingh2002}, R-MAX~\cite{Brafman2001RMAXA}, and metric-$E^3$~\cite{KakadeKL03}, but was not used to encourage pessimism. 
Similar ideas have been explored in related settings like online RL~\citep{Jiang18, vemula2020planning} and imitation learning~\cite{Ainsworth19}.
Our work differs in its focus on offline RL, where we show the P-MDP construction plays a crucial role. Moreover, direct practical instantiations of $E^3$ and metric-$E^3$ with function approximation have remained elusive.

\section{Problem Formulation}\label{sec:formulation}
A {\bf Markov Decision Process (MDP)} is represented by $\M=\{S, A, r, P, \rho_0, \gamma\}$, where, $S$ is the state-space, $A$ is the action-space, $r:S\times A\to [-R_{\max}, R_{\max}]$ is the reward function, $P:S\times A\times S \to \mathbb{R}_+$ is the transition kernel, $\rho_0$ is the initial state distribution, and $\gamma$ the discount factor. A policy defines a mapping from states to a probability distribution over actions, $\pi: S \times A \to \mathbb{R}_+$. %
The goal is to obtain a policy that maximizes expected performance with states sampled according to $\rho_0$, i.e.:\vspace*{-3pt}
\begin{equation}
    \label{eq:RLObjective}
    \max_\pi \ \  J_{\rho_0}(\pi, \M) := \mathbb{E}_{s \sim \rho_0} \left[ V^\pi(s, \M) \right], {\text{ where, \ }} V^\pi(s, \M) = \mathbb{E}\left[\sum_{t=0}^\infty \gamma^t r(s_t,a_t) | s_0 = s\right].
\end{equation}
To avoid notation clutter, we suppress the dependence on $\rho_0$ when understood from context, i.e. $J(\pi, \M) \equiv J_{\rho_0}(\pi, \M)$. We denote the optimal policy using $\pi^* := \arg \max_\pi J_{\rho_0}(\pi, \M)$. Typically, a class of parameterized policies $\pi_\theta \in \Pi(\Theta)$ are considered, and the parameters $\theta$ are optimized.

In {\bf offline RL}, we are provided with a static dataset of interactions with the environment consisting of $\data = \{ (s_i, a_i, r_i, s_i') \}_{i=1}^N$. The data can be collected using one or more logging (or behavioral) policies denoted by $\pi_b$. We do not assume logging policies are known in our formulation. %
Given $\data$, the goal in offline RL is to output a $\pi_{\text{out}}$ with minimal sub-optimality, i.e. $J(\pi^*, \M) - J(\pi_{\text{out}}, \M)$. In general, it may not be possible to learn the optimal policy with a static dataset (see section~\ref{ssec:theory}). Thus, we aim to design algorithms that would result in as low sub-optimality as possible.

{\bf Model-Based RL (MBRL)} involves learning an MDP $\hat{\M} = \{ S, A, r, \hat{P}, \hat{\rho_0}, \gamma \}$ which uses the learned transitions $\hat{P}$ instead of the true transition dynamics $P$. In this paper, we assume the reward function $r$ is known and use it in $\hat{M}$. If $r(\cdot)$ is unknown, it can also be learned from data. The initial state distribution $\hat{\rho_0}$ can either be learned from the data or $\rho_0$ can be used if known. Analogous to $\M$, we use $J_{\hat{\rho_0}}(\pi, \hat{\M})$ or simply $J(\pi, \hat{\M})$ to denote performance of $\pi$ in $\hat{M}$.

\section{Algorithmic Framework}\label{sec:framework}
For ease of exposition and clarity, we first begin by presenting an idealized version of~\fname, for which we also establish theoretical guarantees. Subsequently, we describe a practical version of~\fname~that we use in our experiments. Algorithm~\ref{alg:mbpm} presents the broad framework of~\fname. We now study each component of~\fname~in greater detail.

\begin{algorithm}[h!]
\caption{~\fname: Model Based Offline Reinforcement Learning}
\label{alg:mbpm}
\begin{algorithmic}[1]
\STATE  {\bf Require} Dataset $\mathcal{D}$
\STATE	Learn approximate dynamics model $\hat{P}:S\times A\to S$ using $\mathcal{D}$.
\STATE	Construct $\alpha$-USAD, $U^\alpha:S\times A\to \{\textrm{TRUE},\textrm{FALSE}\}$ using $\mathcal{D}$ \ (see Definition~\ref{def:usad}).
\STATE  Construct the {\em pessimistic} MDP $\Mp=\{S\cup\textrm{HALT},A,r_p,\hat{P}_p,\hat{\rho}_0,\gamma\}$ \ (see Definition~\ref{def:pessimistic_mdp}).
\STATE  (OPTIONAL) Use a behavior cloning approach to estimate the behavior policy $\hat{\pi}_b$.
\STATE  $\pi_{\text{out}}\leftarrow \textrm{PLANNER}(\Mp,\pi_{\text{init}}=\hat{\pi}_b)$
\STATE  {\bf Return} $\pi_{\text{out}}$.
\end{algorithmic}
\end{algorithm}

{\bf Learning the dynamics model:} The first step involves using the offline dataset to learn an approximate dynamics model $\hat{P}(\cdot|s,a)$. This can be achived through maximum likelihood estimation or other techniques from generative and dynamics modeling~\cite{Venkatraman2015DAD, Bengio2015ScheduledSF, Vaswani2017AttentionIA}.
Since the offline dataset may not span the entire state space, the learned model may not be globally accurate. So, a na\"ive MBRL approach that directly plans with the learned model may over-estimate rewards in unfamiliar parts of the state space,
resulting in a highly sub-optimal policy~\cite{RossB12}. We overcome this with the next step.
    
{\bf Unknown state-action detector (USAD):} We partition the state-action space into known and unknown regions based on the accuracy of learned model as follows.
\begin{definition} 
\label{def:usad}
($\alpha$-USAD)
Given a state-action pair $(s,a)$, define an unknown state action detector as:%
\begin{equation}
\begin{small}
    U^\alpha(s,a) = 
    \begin{cases} 
        \textrm{FALSE \ (i.e. Known)} & \mathrm{if} \  \ D_{TV} \left( \hat{P}(\cdot|s,a), P(\cdot|s,a) \right) \leq \alpha \text{ \ can be guaranteed} \\
        \textrm{TRUE \ \ (i.e. Unknown)} & \mbox{otherwise}
    \end{cases}
\end{small}
\end{equation}%
\end{definition}\vspace*{-7pt}
Here $D_{TV} \left( \hat{P}(\cdot|s,a), P(\cdot|s,a) \right)$ denotes the total variation distance between $\hat{P}(\cdot|s,a)$ and $P(\cdot|s,a)$.
Intuitively, USAD provides confidence about where the learned model is accurate. It flags state-actions for which the model is guarenteed to be accurate as ``known'', while flagging state-actions where such a guarantee cannot be ascertained as ``unknown''. Note that USAD is based on the ability to guarantee the accuracy, and is not an inherent property of the model. In other words, there could be states where the model is actually accurate, but flagged as unknown due to the agent's inability to guarantee accuracy. Two factors contribute to USAD's effectiveness: (a) data availability: having sufficient data points ``close'' to the query; (b) quality of representations: certain representations, like those based on physics, can lead to better generalization guarantees. This suggests that larger datasets and research in representation learning can potentially enable stronger offline RL results.

{\bf Pessimistic MDP construction:} We now construct a pessimistic MDP (P-MDP) using the learned model and USAD, which penalizes policies that venture into unknown parts of state-action space.
\begin{definition}
\label{def:pessimistic_mdp}
The $(\alpha, \nrew)$-pessimistic MDP is described by $\Mp := \{S\cup\textrm{HALT}, A, r_p, \hat{P}_p, \hat{\rho}_0, \gamma\}$. Here, $S$ and $A$ are states and actions in the MDP $\M$. HALT is an additional absorbing state we introduce into the state space of $\Mp$. $\hat{\rho}_0$ is the initial state distribution learned from the dataset $\data$. $\gamma$ is the discount factor (same as $\M$). The modified reward and transition dynamics are given by:
\begin{equation*}
\begin{split}
    \hat{P}_{p}(s'|s,a) = 
    \begin{cases}
    \delta(s'=\mathrm{HALT}) & 
    \begin{array}{l}
            \mathrm{if} \  U^\alpha(s,a) = \mathrm{TRUE}   \\
          \mathrm{or} \ \ s = \mathrm{HALT}
    \end{array} \\
    \hat{P}(s'|s,a) & \mbox{\ \ otherwise}
    \end{cases}
\end{split}
\quad
\begin{split}
    r_p(s,a) = 
    \begin{cases}
    -\nrew & \mathrm{if} \  s = \mathrm{HALT} \\
    r(s,a) & \mbox{otherwise}
    \end{cases}
\end{split}
\end{equation*}
\end{definition}\vspace*{-3mm}
$\delta(s'=\mathrm{HALT})$ is the Dirac delta function, which forces the MDP to transition to the absorbing state $\mathrm{HALT}$. For unknown state-action pairs, we use a reward of $-\nrew$, while all known state-actions receive the same reward as in the environment. The P-MDP heavily punishes policies that visit unknown states, thereby providing a safeguard against distribution shift and model exploitation.

{\bf Planning:} The final step in~\fname~is to perform planning in the P-MDP defined above. For simplicity, we assume a planning oracle that returns an $\epsilon_{\pi}$-sub-optimal policy in the P-MDP. A number of algorithms based on MPC~\cite{Todorov2005, WilliamsMPPI}, search-based planning~\cite{RRT, MCTSSurvey}, dynamic programming~\cite{MnihKSRVBGRFOPB15, Munos2008FiniteTimeBF}, or policy optimization~\cite{SchulmanPPO, HaarnojaSAC, Rajeswaran17nips, FakoorP30} can be used to approximately realize this..

\subsection{Theoretical Results}\label{ssec:theory}
In order to state our results, we begin by defining the notion of hitting time.
\begin{definition}(Hitting time)
Given an MDP $\M$, starting state distribution $\rho_0$, state-action pair $(s,a)$ and a policy $\pi$, the {\em hitting time} $T_{(s,a)}^{\pi}$ is defined as the random variable denoting the first time action $a$ is taken at state $s$ by $\pi$ on $\M$, and is equal to $\infty$ if $a$ is never taken by $\pi$ from state $s$. For a set of state-action pairs $\mathcal{S} \subseteq S \times A$, we define $T_{\mathcal{S}}^{\pi} \defeq \min_{(s,a) \in \mathcal{S}}T_{(s,a)}^{\pi}$.
\end{definition}
\vspace*{-5pt}
We are now ready to present our main result with the proofs deferred to the appendix.
\begin{theorem}\label{thm:pess-world}(Policy value with pessimism)
The value of any policy $\pi$ on the original MDP $\M$ and its $(\alpha,R_{\max})$-pessimistic MDP $\Mp$ satisfies:\vspace*{-1mm}%
\begin{align*}
    J_{\hat{\rho_0}}(\pi, \Mp) &\geq J_{\rho_0}(\pi, \M) - \frac{2\Rmax}{1-\gamma} \cdot D_{TV}(\rho_0, \hat{\rho_0}) - \frac{2 \gamma  \Rmax}{(1-\gamma)^2} \cdot \alpha - \frac{2 \Rmax}{1-\gamma} \cdot \E{\gamma^{T_{\U}^\pi}}, \mbox{ and } \\
    J_{\hat{\rho_0}}(\pi, \Mp) &\leq J_{\rho_0}(\pi, \M) + \frac{2\Rmax}{1-\gamma} \cdot D_{TV}(\rho_0, \hat{\rho_0}) + \frac{2 \gamma \Rmax}{(1-\gamma)^2}\cdot \alpha,
\end{align*}\vspace*{-5pt}
where $T_{\U}^\pi$ denotes the hitting time of unknown states $\U \defeq \set{(s,a): U^{\alpha}(s,a) = \textrm{TRUE}}$ by $\pi$ on $\M$.
\end{theorem}

Theorem~\ref{thm:pess-world} can be used to bound the suboptimality of output policy $\pi_{\textrm{out}}$ of Algorithm~\ref{alg:mbpm}.
\begin{corollary}\label{cor:subopt}
Suppose PLANNER in Algorithm~\ref{alg:mbpm} returns an $\epsilon_{\pi}$ sub-optimal policy. Then, we have
$$J_{\rho_0}(\pi^*, \M) - J_{\rho_0}(\pi_{\textrm{out}}, \M) \leq \epsilon_{\pi} + \frac{4 \Rmax}{1-\gamma} \cdot D_{TV}(\rho_0, \hat{\rho_0}) + \frac{4\gamma \Rmax}{(1-\gamma)^2} \cdot \alpha + \frac{2 \Rmax}{1-\gamma} \cdot \E{\gamma^{T_{\U}^{\pi^*}}}.$$
\end{corollary}
Theorem~\ref{thm:pess-world} indicates that the difference in any policy $\pi$'s value in the $(\alpha,R_{\max})$ pessimistic MDP $\Mp$ and the original MDP $\M$ depends on:
i) the total variation distance between the true and learned starting state distribution $D_{TV}(\rho_0, \hat{\rho_0})$, ii) the maximum total variation distance $\alpha$ between the learned model $\hat{P}(\cdot \vert s,a)$ and the true model $P(\cdot \vert s,a)$ over all {\em known} states i.e., $\set{(s,a) \vert U^{\alpha}(s,a) = \textrm{FALSE}}$ and, iii) the hitting time $T_{\U}^{\pi^*}$ of unknown states $\U$ on the original MDP $\M$ under the optimal policy $\pi^*$. As the dataset size increases, $D_{TV}(\rho_0, \hat{\rho_0})$ and $\alpha$ approach zero, indicating $\E{\gamma^{T_{\U}^{\pi^*}}}$ determines the sub-optimality in the limit. 
For comparison to prior work, Lemma~\ref{lem:disc} in Appendix~\ref{app:proofs} bounds this quantity in terms of state-action visitation distribution, which for a policy $\pi$ on $\M$ is expressed as ${\small d^{\pi,\M}(s,a) \defeq (1-\gamma)\sum_{t=0}^\infty\gamma^t P(s_t=s,a_t=a|s_0\sim\rho_0,\pi,\M)}$.
Furthermore, we can also show that~\fname~learns a policy that improves over the behavioral policy with high probability. The following lemma presents both of these results:
\begin{lemma}\label{lem:asymptotic}(Upper bound; \fname~improves over the behavioral policy)
Suppose $\rho_{0,\textrm{min}}>0$, $p_{\textrm{min}}>0$ and $d^{\pi_b}_\textrm{min} > 0$ are the smallest non-zero elements of initial distribution $\rho_0$, state transition probabilities $P(\cdot|s,a)$, and discounted state probability distribution $d^{\pi_b,\M}(s,a)$ respectively. If the dataset $\data$ consists of $n \geq \frac{C}{\left(d^{\pi_b}_\textrm{min}\right)^2} \cdot \log \frac{1}{\delta d^{\pi_b}_\textrm{min}}$ independent trajectories sampled according to a behavior policy $\pi_b$ with initial distribution $\rho_0$, then the output $\pi_{\textrm{out}}$ of Algorithm~\ref{alg:mbpm} satisfies:
\begin{align*}
    J_{\rho_0}(\pi_b,\M) - J_{\rho_0}(\pi_{\textrm{out}},\M) &\leq \epsilon_{\pi} + \epsilon_n, \mbox{ and } \\
    J_{\rho_0}(\pi^*,\M) - J_{\rho_0}(\pi_{\textrm{out}},\M) &\leq  \epsilon_{\pi} + \frac{2 \Rmax}{1-\gamma} \cdot \E{\gamma^{T_{\U}^{\pi^*}}} + \epsilon_n \leq \epsilon_{\pi} + \frac{2 R_{\max}}{(1-\gamma)^2} \cdot d^{\pi^*, \M}(\U) + \epsilon_n,
\end{align*}
with probability at least $1-C \delta$, where $C$ is a large enough constant and
\begin{align*}
    \epsilon_n \defeq \frac{4C R_\textrm{max}}{(1-\gamma)\rho_{0,\textrm{min}}} \cdot \sqrt{\frac{\log \frac{1}{\delta\rho_{0,\textrm{min}}}}{n}} + \frac{4C \gamma R_\textrm{max}}{(1-\gamma)^2p_{\textrm{min}}} \cdot \sqrt{\frac{\log \frac{1}{\delta p_\textrm{min} d^{\pi_b}_\textrm{min}}}{d^{\pi_b}_\textrm{min} \cdot n}}
\end{align*}
is an error term related to finite samples that goes to zero as $n \rightarrow \infty$.
\end{lemma}

The bound consists of three terms: (i) a sampling error term $\epsilon_n$ which decreases with larger dataset sizes that is typical of offline RL; (ii) an optimization error term $\epsilon_\pi$ that can be made small with additional compute to find the optimal policy in the learned model; and (iii)~a distribution shift term that depends on the coverage of the offline dataset and overlap with the optimal policy.

Prior results~\cite{FujimotoBCQ,LiuSAB19} assume that $d^{\pi^*,\M}(\U_D)=0$, where $\U_D \defeq \set{(s,a) \vert (s,a,r,s') \notin \mathcal{D}} \supseteq \U$ is the set of state action pairs that don't occur in the offline dataset, and guarantee finding an optimal policy under this assumption. Our result significantly improves upon these in three ways:
i)~$\U_D$ is replaced by a smaller set $\U$, leveraging the generalization ability of learned dynamics model,
ii)~the sub-optimality bound
is extended to the setting where full support coverage is not satisfied i.e., $d^{\pi^*,\M}(\U)>0$, and iii)~the sub-optimality bound on $\pi_{\textrm{out}}$ is stated in terms of unknown state hitting time $T_{\U}^{\pi^*}$, which can be significantly better than a bound that depends only on $d^{\pi^*,\M}(\U)$.
To further strengthen our results, the following proposition shows that Lemma~\ref{lem:asymptotic} is tight up to $\log$ factors.

\begin{proposition}\label{prop:counter}(Lower bound)
For any discount factor $\gamma \in [0.95,1)$, support mismatch $\epsilon \in \left(0,\frac{1-\gamma}{\log \frac{1}{1-\gamma}}\right]$ and reward range $[-\Rmax,\Rmax]$, there is an MDP $\M$, starting state distribution $\rho_0$, optimal policy $\pi^*$ and a dataset collection policy $\pi_b$ such that
i) $d^{\pi^*,\M}(\U_D) \leq \epsilon$, and
ii) any policy $\hat{\pi}$ that is learned solely using the dataset collected with $\pi_b$ satisfies:
\begin{align*}%
\begin{small}
    J_{\rho_0}(\pi^*,\mathcal{M}) - J_{\rho_0}(\hat{\pi}, \mathcal{M}) \geq \frac{\Rmax}{4(1-\gamma)^2} \cdot \frac{\epsilon}{\log \frac{1}{1-\gamma}},
\end{small}
\end{align*}
where $\U_D \defeq \set{(s,a) : (s,a,r,s') \notin \mathcal{D} \mbox{ for any } r,s'}$ denotes state action pairs not in the dataset $\mathcal{D}$.
\end{proposition}

We see that for $\epsilon < (1-\gamma)/(\log \tfrac{1}{1-\gamma})$, the lower bound obtained by Proposition~\ref{prop:counter} on the suboptimality of any offline RL algorithm matches the asymptotic (as $n \rightarrow \infty$) upper bound of Lemma~\ref{lem:asymptotic} up to an additional log factor. For $\epsilon > (1-\gamma)/(\log \tfrac{1}{1-\gamma})$, Proposition~\ref{prop:counter} also implies (by choosing $\epsilon' = (1-\gamma)/(\log \tfrac{1}{1-\gamma}) < \epsilon$) that any offline algorithm must suffer at least constant factor suboptimality in the worst case. Finally, we note that as the size of dataset $\mathcal{D}$ increases to $\infty$, Theorem~\ref{thm:pess-world} and the optimality of PLANNER (i.e., $\epsilon_\pi=0$) together imply that $J_{\rho_0}(\pi_{\textrm{out}},\M) \geq J_{\rho_0}(\pi_b,\M)$.

\subsection{Practical Implementation Of~\fname}\label{ssec:practicalInstantiation}
We now present a practical instantiation of~\fname~(algorithm~\ref{alg:mbpm}) utilizing a recent model-based NPG approach~\citep{RajeswaranGameMBRL}. The principal difference is the specialization to offline RL and construction of the P-MDP using an ensemble of learned dynamics models.

{\bf Dynamics model learning:} We consider Gaussian dynamics models~\citep{RajeswaranGameMBRL}, $\hat{P}(\cdot|s,a) \equiv \mathcal{N} \left( f_\phi(s,a), \Sigma \right)$, with mean $f_\phi(s,a) = s + \sigma_\Delta \ \textrm{MLP}_\phi \left( (s - \mu_s)/\sigma_s, (a - \mu_a)/\sigma_a \right)$, where $\mu_s, \sigma_s, \mu_a, \sigma_a$ are the mean and standard deviations of states/actions in $\data$; $\sigma_\Delta$ is the standard deviation of state differences, i.e. $\Delta = s' - s, (s,s')\in\data$; this parameterization ensures local continuity since the MLP learns only the state differences. The MLP parameters are optimized using maximum likelihood estimation with mini-batch stochastic optimization using Adam~\cite{KingmaB14}.

{\bf Unknown state-action detector (USAD):} In order to partition the state-action space into known and unknown regions, we use uncertainty quantification~\cite{OsbandAC18, Azizzadenesheli18, BurdaESK19, POLO}. In particular, we consider approaches that track uncertainty using the predictions of ensembles of models~\cite{OsbandAC18, POLO}. We learn multiple models $\{ f_{\phi_1}, f_{\phi_2}, \ldots \}$ where each model uses a different weight initialization and are optimized with different mini-batch sequences. Subsequently, we compute the ensemble discrepancy as $\mathrm{disc}(s,a) = \max_{i, j} \ \left\| f_{\phi_i}(s,a)-f_{\phi_j}(s,a) \right\|_2$, where $f_{\phi_i}$ and $f_{\phi_j}$ are members of the ensemble. With this, we implement USAD as below, with $\mathrm{threshold}$ being a tunable hyperparameter.
\begin{equation}
    U_{\text{practical}}(s,a) = 
    \begin{cases} 
        \textrm{FALSE \ (i.e. Known)} & \mathrm{if} \  \ \mathrm{disc}(s,a) \leq \mathrm{threshold} \\
        \textrm{TRUE \ (i.e. Unknown)} & \mathrm{if} \  \ \mathrm{disc}(s,a) > \mathrm{threshold}
    \end{cases}.
\end{equation}

\section{Experiments}\label{sec:experiments}
Through our experimental evaluation, we aim to answer the following questions:
\begin{enumerate}[leftmargin=*]
    \itemsep0em
    \item {\bf Comparison to prior work:} How does~\fname~compare to prior SOTA offline RL algorithms~\citep{FujimotoBCQ,KumarBEAR,WuTN19BRAC} in commonly studied benchmark tasks?
    \item {\bf Quality of logging policy:} How does the quality (value) of the data logging (behavior) policy, and by extension the dataset, impact the quality of the policy learned by~\fname?
    \item {\bf Importance of pessimistic MDP:} %
    How does~\fname~compare against a na\"ive model-based RL approach that directly plans in a learned model without any safeguards?
    \item {\bf Transfer from pessimistic MDP to environment:} Does learning progress in the P-MDP, which we use for policy learning, effectively translate or transfer to learning progress in the environment?
\end{enumerate}

To answer the above questions, we consider commonly studied benchmark tasks from OpenAI gym~\cite{Brockman16} simulated with MuJoCo~\cite{TodorovET12}. Our experimental setup closely follows prior work~\citep{FujimotoBCQ, KumarBEAR, WuTN19BRAC}. The tasks considered include~\hopper,~\hc,~\ant,~and \walker, which are illustrated in Figure~\ref{fig:gym_tasks_illustration}. We consider five different logged data-sets for each environment, totalling 20 environment-dataset combinations. Datasets are collected based on the work of~\citet{WuTN19BRAC}, with each dataset containing the equivalent of 1 million timesteps of environment interaction. We first partially train a policy $(\pi_p)$ to obtain values around 1000, 4000, 1000, and 1000 respectively for the four environments. The first exploration strategy, \texttt{Pure}, involves collecting the dataset solely using $\pi_p$. The four other datasets are collected using a combination of $\pi_p$, a {\em noisy} variant of $\pi_p$, and an untrained random policy. The noisy variant of $\pi_p$ utilizes either epsilon-greedy or Gaussian noise, resulting in configurations $\texttt{eps-1},\texttt{eps-3},\texttt{gauss-1},\texttt{gauss-3}$ that signify various types and magnitudes of noise added to $\pi_p$. Please see appendix for additional experimental details.

\begin{figure}[b!]
    \centering
    \includegraphics[width=\textwidth]{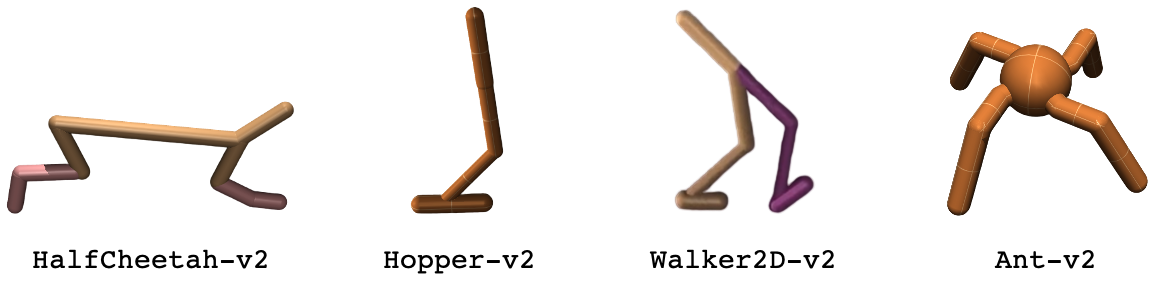}
    \caption{Illustration of the suite of tasks considered in this work. These tasks require the RL agent to learn locomotion gaits for the illustrated simulated characters.}
    \vspace*{-5pt}
    \label{fig:gym_tasks_illustration}
\end{figure}

We parameterize the dynamics model using 2-layer ReLU-MLPs and use an ensemble of 4 dynamics models to implement USAD as described in Section~\ref{ssec:practicalInstantiation}. We parameterize the policy using a 2-layer tanh-MLP, and train it using model-based NPG~\citep{RajeswaranGameMBRL}. We evaluate the learned policies using rollouts in the (real) environment, but these rollouts are not made available to the algorithm in any way for purposes of learning. This is similar to evaluation protocols followed in prior work~\cite{WuTN19BRAC, FujimotoBCQ, KumarBEAR}. We present all our results averaged over $5$ different random seeds. Note that we use the same hyperparameters for all random seeds. In contrast, the prior works whose results we compare against tune hyper-parameters separately for each random seed~\cite{WuTN19BRAC}. 


\begin{table}[t!]
\vspace*{-20pt}
\caption{Results in various environment-exploration combinations. Baselines are reproduced from~\citet{WuTN19BRAC}. Prior work does not provide error bars. For MOReL results, error bars indicate the standard deviation across 5 different random seeds. We choose SOTA result based on the average performance.}\vspace{10pt}
\label{table:all_results_1}
\begin{minipage}{.49\linewidth}
\begin{adjustbox}{width=\columnwidth,center}
\begin{tabular}{l|ccc|c||c}
\toprule
\multicolumn{6}{c}{\bf Environment: Ant-v2}\\ %
\midrule
Algorithm &
\vtop{\hbox{\strut BCQ}\hbox{\strut\cite{FujimotoBCQ}}} &
\vtop{\hbox{\strut BEAR}\hbox{\strut\cite{KumarBEAR}}} &
\vtop{\hbox{\strut {\scriptsize BRAC}}\hbox{\strut\cite{WuTN19BRAC}}} &
\vtop{\hbox{\strut Best}\hbox{\strut Baseline}}  &
\vtop{\hbox{\strut \fname}\hbox{\strut (Ours)}} 
\\
\midrule
\texttt{Pure} & 1921 & 2100 & \underline{2839} & 2839 & {\bf 3663$\pm$247}\\
\texttt{Eps-1} & 1864 & 1897 & \underline{2672} & 2672 & {\bf 3305$\pm$413}\\
\texttt{Eps-3} & 1504 & 2008 & \underline{2602} & 2602 & {\bf 3008$\pm$231}\\
\texttt{Gauss-1} & 1731 & 2054 & \underline{2667} & 2667 & {\bf 3329$\pm$270}\\
\texttt{Gauss-3} & 1887 & 2018 & 2640 & 2661 & {\bf 3693$\pm$33}\\
\bottomrule
\end{tabular}
\end{adjustbox}
\end{minipage}
\hspace*{5pt}
\begin{minipage}{.49\linewidth}
\begin{adjustbox}{width=\columnwidth,center}
\begin{tabular}{l|ccc|c||c}
\toprule
\multicolumn{6}{c}{\bf Environment: Hopper-v2}\\ %
\midrule
Algorithm &
\vtop{\hbox{\strut BCQ}\hbox{\strut\cite{FujimotoBCQ}}} &
\vtop{\hbox{\strut BEAR}\hbox{\strut\cite{KumarBEAR}}} &
\vtop{\hbox{\strut {\scriptsize BRAC}}\hbox{\strut\cite{WuTN19BRAC}}} &
\vtop{\hbox{\strut Best}\hbox{\strut Baseline}}  &
\vtop{\hbox{\strut \fname}\hbox{\strut (Ours)}} 
\\
\midrule
\texttt{Pure} & 1543 & 0 & 2291 & 2774 & {\bf 3642$\pm$54}\\
\texttt{Eps-1} & 1652 & 1620 & 2282 & 2360 & {\bf 3724$\pm$46}\\
\texttt{Eps-3} & 1632 & 2213 & 1892 & 2892 & {\bf 3535$\pm$91}\\
\texttt{Gauss-1} & 1599 & 1825 & \underline{2255} & 2255 & {\bf 3653$\pm$52}\\
\texttt{Gauss-3} & 1590 & 1720 & 1458 & 2097 & {\bf 3648$\pm$148}\\
\bottomrule
\end{tabular}
\end{adjustbox}
\end{minipage}
\vspace{10pt}
\begin{minipage}{.49\linewidth}
\begin{adjustbox}{width=\columnwidth,center}
\begin{tabular}{l|ccc|c||c}
\toprule
\multicolumn{6}{c}{\bf Environment: HalfCheetah-v2}\\
\midrule
Algorithm &
\vtop{\hbox{\strut BCQ}\hbox{\strut\cite{FujimotoBCQ}}} &
\vtop{\hbox{\strut BEAR}\hbox{\strut\cite{KumarBEAR}}} &
\vtop{\hbox{\strut {\scriptsize BRAC}}\hbox{\strut\cite{WuTN19BRAC}}} &
\vtop{\hbox{\strut Best}\hbox{\strut Baseline}}  &
\vtop{\hbox{\strut \fname}\hbox{\strut (Ours)}} 
\\
\midrule
\texttt{Pure} & 5064 & 5325 & 6207 & {\bf 6209} & {\bf 6028$\pm$192}\\
\texttt{Eps-1} & 5693 & 5435 & \underline{\bf 6307} & {\bf 6307} & {5861$\pm$192}\\
\texttt{Eps-3} & 5588 & 5149 & 6263 & {\bf 6359} & {5869$\pm$139}\\
\texttt{Gauss-1} & 5614 & 5394 & \underline{\bf 6323} & {\bf 6323} & {6026$\pm$74}\\
\texttt{Gauss-3} & 5837 & 5329 & \underline{\bf 6400} & {\bf 6400} & {5892$\pm$128}\\
\bottomrule
\end{tabular}
\end{adjustbox}
\end{minipage}
\hspace*{5pt}
\begin{minipage}{.49\linewidth}
\begin{adjustbox}{width=\columnwidth,center}
\begin{tabular}{l|ccc|c||c}
\toprule
\multicolumn{6}{c}{\bf Environment: Walker-v2}\\ %
\midrule
Algorithm &
\vtop{\hbox{\strut BCQ}\hbox{\strut\cite{FujimotoBCQ}}} &
\vtop{\hbox{\strut BEAR}\hbox{\strut\cite{KumarBEAR}}} &
\vtop{\hbox{\strut {\scriptsize BRAC}}\hbox{\strut\cite{WuTN19BRAC}}} &
\vtop{\hbox{\strut Best}\hbox{\strut Baseline}}  &
\vtop{\hbox{\strut \fname}\hbox{\strut (Ours)}} 
\\
\midrule
\texttt{Pure} & 2095 & 2646 & 2694 & 2907 & {\bf 3709$\pm$159}\\
\texttt{Eps-1} & 1921 & 2695 & 3241 & {\bf 3490} & {2899$\pm$588}\\
\texttt{Eps-3} & 1953 & 2608 & \underline{\bf 3255} & {\bf 3255} & {\bf 3186$\pm$92}\\
\texttt{Gauss-1} & 2094 & 2539 & 2893 & 3193 & {\bf 4027$\pm$314}\\
\texttt{Gauss-3} & 1734 & 2194 & \underline{\bf 3368} & {\bf 3368} & {\bf 2828$\pm$589}\\
\bottomrule
\end{tabular}
\end{adjustbox}
\end{minipage}\vspace*{-7mm}
\label{table:benchmark_comparisons}
\end{table}

\paragraph{Comparison of~\fname's performance with prior work}
\label{ssec:otherExploration}
We compare results of~\fname~with prior SOTA algorithms like BCQ, BEAR, and all variants of BRAC. The results are summarized in Table~\ref{table:all_results_1}. For fairness of comparison, we reproduce results from prior work and do not run the algorithms ourselves. We provide a more expansive table with additional baseline algorithms in the appendix. Our algorithm, \fname, achives SOTA results in $12$ out of the $20$ environment-dataset combinations, overlaps in error bars for $3$ other combinations, and is competitive in the remaining cases. In contrast, the next best approach (a variant of BRAC) achieves SOTA results in only $5$ out of $20$ configurations.

\paragraph{Comparison of~\fname's performance in the D4RL benchmark suite} The D4RL benchmark suite~\cite{Fu2020D4RLDF} for offline RL was introduced in concurrent work. We also study the performance of~\fname~in this benchmark suite. We find that~\fname~achieves the highest (normalized) score in $5$ out of $12$ domains studied, while the next best algorithm~(CQL) achieves the highest score in only $3$ out of $12$ domains. Furthermore, we observe that~\fname~is often very competitive with the best performing algorithm in any given domain even if it doesn't achieve the top score. However, in many domains,~\fname~significantly improves over the state of the art (e.g. hopper-medium-replay and hopper-random). To aggregate results across multiple domains, we consider the average of the normalized scores as a proxy, and observe that~\fname~significantly outperforms prior algorithms. 

\begin{table*}[h!]
\label{table:d4rl_results}
\small
\caption{Results of various algorithms on the D4RL benchmark suite. Each number is the normalized score computed as (score $-$ random policy score) $/$ (expert policy score $-$ random policy score). The raw score for~\fname~was taken to be the average over the last $100$ iterations of policy learning averaged over $3$ random seeds. Results of MOPO~\cite{MOPO} and CQL~\cite{Kumar2020CQL} are reported from their respective papers. Remaining results are taken from the D4RL benchmark suite white-paper~\cite{Fu2020D4RLDF}. \\ }
\begin{adjustbox}{max width=\textwidth}
\begin{tabular}{@{}|l|l|c|c|c|c|c|c|c|@{}}
\toprule
\multicolumn{1}{|c|}{\textbf{Dataset}} & \multicolumn{1}{c|}{\textbf{Environment}} & \textbf{\begin{tabular}[c]{@{}c@{}}MOReL\\ (Ours)\end{tabular}} & \textbf{MOPO} & \textbf{CQL}  & \textbf{SAC-Off} & \textbf{BEAR} & \textbf{BRAC-p} & \textbf{BRAC-v} \\ \midrule
random                               & halfcheetah                              & 25.6                                                            & 34.4          & \textbf{35.4} & 30.5             & 25.1          & 24.1            & 31.2            \\
random                               & hopper                                   & \textbf{53.6}                                                   & 11.7          & 10.8          & 11.3             & 11.4          & 11              & 12.2            \\
random                               & walker2d                                 & \textbf{37.3}                                                   & 13.6          & 7             & 4.1              & 7.3           & -0.2            & 1.9             \\
medium                               & halfcheetah                              & 42.1                                                            & 42.3          & 44.4          & -4.3             & 41.7          & 43.8            & \textbf{46.3}   \\
medium                               & hopper                                   & \textbf{95.4}                                                   & 28.0          & 86.6          & 0.8              & 52.1          & 32.7            & 31.1            \\
medium                               & walker2d                                 & 77.8                                                            & 17.8          & 74.5          & 0.9              & 59.1          & 77.5            & \textbf{81.1}   \\
medium-replay                        & halfcheetah                              & 40.2                                                            & \textbf{53.1} & 46.2          & -2.4             & 38.6          & 45.4            & 47.7            \\
medium-replay                        & hopper                                   & \textbf{93.6}                                                   & 67.5          & 48.6          & 3.5              & 33.7          & 0.6             & 0.6             \\
medium-replay                        & walker2d                                 & \textbf{49.8}                                                   & 39.0          & 32.6          & 1.9              & 19.2          & -0.3            & 0.9             \\
medium-expert                        & halfcheetah                              & 53.3                                                            & \textbf{63.3} & 62.4          & 1.8              & 53.4          & 44.2            & 41.9            \\
medium-expert                        & hopper                                   & 108.7                                                           & 23.7          & \textbf{111}  & 1.6              & 96.3          & 1.9             & 0.8             \\
medium-expert                        & walker2d                                 & 95.6                                                            & 44.6          & \textbf{98.7} & -0.1             & 40.1          & 76.9            & 81.6            \\ \midrule
Average                              & Average                                  & \textbf{64.42}                                                  & 36.58         & 54.85         & 4.13             & 39.83         & 29.80           & 31.44           \\ \bottomrule
\end{tabular}
\end{adjustbox}
\end{table*}


\clearpage
\textbf{Importance of Pessimistic MDP} 
\label{ssec:ablation}
To highlight the importance of P-MDP, we again consider the \texttt{Pure-partial} dataset outlined above. We compare~\fname~with a nai\"ve MBRL approach that first learns a dynamics model using the offline data, followed by running model-based NPG without any safeguards against model inaccuracy. The results are summarized in Figure~\ref{fig:morel_vs_naive}. We observe that the nai\"ve MBRL approach already works well, achieving results comparable to prior algorithms like BCQ and BEAR. However,~\fname~clearly exhibits more stable and monotonic learning progress. This is particularly evident in \hopper,~\hc, and \walker, where an uncoordinated set of actions can result in the agent falling over. Furthermore, in the case of nai\"ve MBRL, we observe that performance can quickly degrade after a few hundred steps of policy improvement, such as in case of \hopper,~\hc~and~\walker. This suggests that the learned model is being over-exploited. In contrast, with~\fname, we observe that the learning curve is stable and nearly monotonic even after many steps of policy improvement. 
\begin{figure}[t!]
	\includegraphics[width=\textwidth]{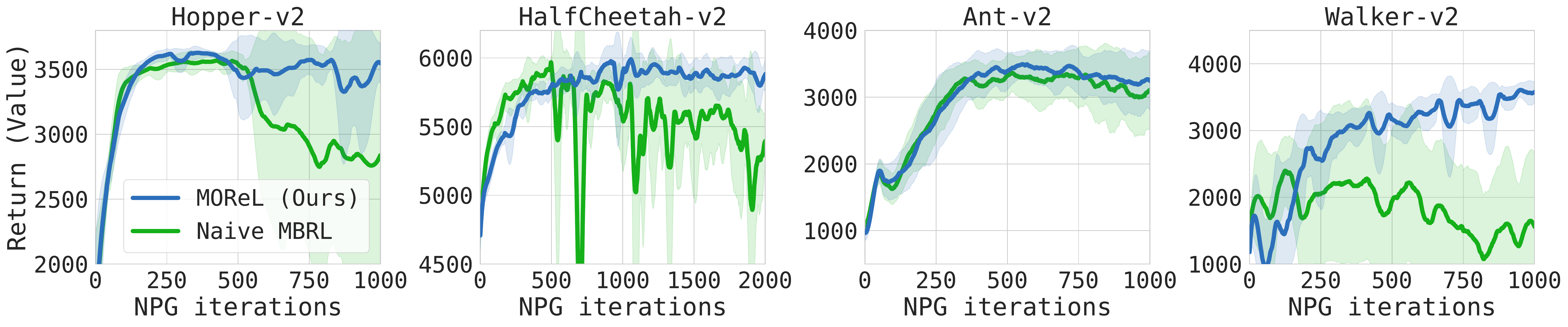}
	\caption{\fname~and Naive MBRL learning curves. The x-axis plots the number of model-based NPG iterations, while y axis plots the return (value) in the real environment. The naive MBRL algorithm is highly unstable while~\fname~leads to stable and near-monotonic learning. Notice however that even naive MBRL learns a policy that performs often as well as the best model-free offline RL algorithms.}
	\label{fig:morel_vs_naive} 
\end{figure}


\paragraph{Quality of logging policy}
\label{ssec:loggingPolicyDependence}
\begin{wraptable}{r}{0.55\linewidth}%

	\caption{Value of the policy learned by~\fname~(5 random seeds) when working with a dataset collected with a random (untrained) policy (\texttt{Pure-random}) and a partially trained policy (\texttt{Pure-partial}). 
	}
	\label{table:RandomVsSubopt}
	\begin{adjustbox}{width=0.55\columnwidth,center}
	\begin{tabular}{l | c | c}
		\toprule
		Environment & \texttt{Pure-random} & \texttt{Pure-partial}\\[1ex]\midrule
		\hopper& $2354\pm443$ & $3642\pm54$\\[0.5ex]
		\hc& $2698\pm230$ &  $6028\pm192$\\[0.5ex]
		\walker& $1290\pm325$ & $3709\pm159$\\[0.5ex]
		\ant& $1001\pm3$ & $3663\pm247$\\[1ex]
		\bottomrule
	\end{tabular}
	\end{adjustbox}
	\vspace*{-10pt}
\end{wraptable}
Section~\ref{ssec:theory} indicates that it is not possible for any offline RL algorithm to learn a near-optimal policy when faced with support mismatch between the dataset and optimal policy. To verify this experimentally for~\fname, we consider two datasets (of the same size) collected using the \texttt{Pure} strategy. The first uses a partially trained policy $\pi_p$ (called \texttt{Pure-partial}), which is the same as the \texttt{Pure} dataset studied in Table~\ref{table:benchmark_comparisons}. The second dataset is collected using an untrained random Gaussian policy (called \texttt{Pure-random}). Table~\ref{table:RandomVsSubopt} compares the results of~\fname~using these two datasets.
We observe that the value of policy learned with \texttt{Pure-partial} dataset far exceeds the value with the \texttt{Pure-random} dataset. Thus, the value of policy used for data logging plays a crucial role in the performance achievable with offline RL. \\[0.2cm]


\begin{figure}[t!]
	\includegraphics[width=\textwidth]{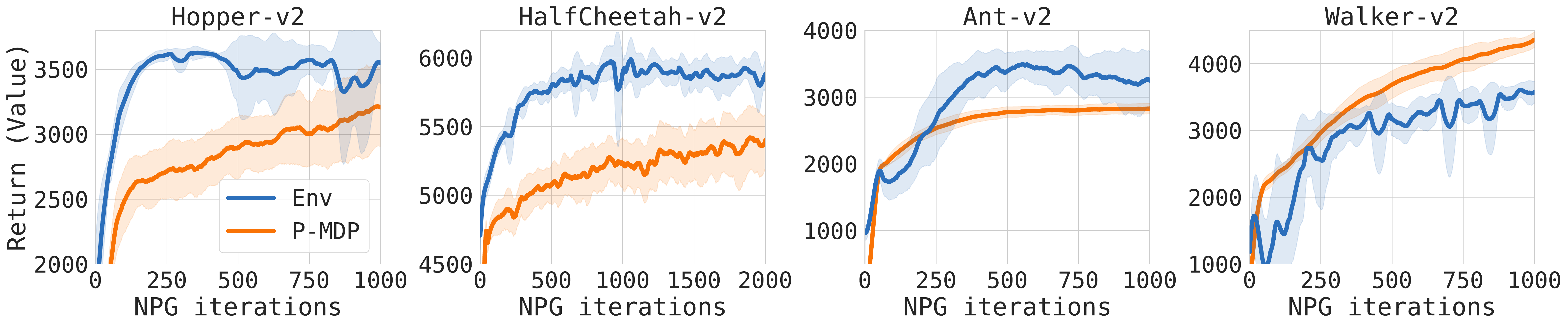}
	\caption{Learning curve using the \texttt{Pure-partial} dataset, see paper text for details. The policy is learned using the pessimistic MDP (P-MDP), and we plot the performance in both the P-MDP and the real environment over the course of learning. We observe that the performance in the P-MDP closely tracks the true performance and never substantially exceeds it, as predicted in section~\ref{ssec:theory}. This shows that the policy value in the P-MDP serves as a good surrogate for the purposes of offline policy evaluation and learning.}
	\label{fig:pessimism_lowerbound} 
\end{figure}

\textbf{Transfer from P-MDP to environment} \\[0.2cm]
Finally, we study how the learning progress in P-MDP relates to the progress in the environment. Our theoretical results (Theorem~1) suggest that the value of a policy in the P-MDP cannot substantially exceed the value in the environment. This makes the value in the P-MDP an approximate lower bound on the true performance, and a good surrogate for optimization.  In Figure~\ref{fig:pessimism_lowerbound}, we plot the value or return of the policy in the P-MDP and environment over the course of learning. Note that the policy is being learned in the P-MDP, and as a result we observe a clear monotonic learning curve for value in the P-MDP, consistent with the monotonic improvement theory of policy gradient methods~\cite{Kakade2002CPI, SchulmanTRPO}. We observe that the value in the true environment closely correlates with the value in P-MDP. In particular, the P-MDP value never substantially exceeds the true performance, suggesting that the pessimism helps to avoid model exploitation.

\section{Conclusions}\label{sec:discussion}

We introduced~\fname, a new model-based framework for offline RL. \fname~incorporates both \emph{generalization} and \emph{pessimism} (or conservatism). This enables \fname~to perform policy improvement in {known states} that may not directly occur in the static offline dataset, but can nevertheless be predicted using the dataset by leveraging the power of generalization. At the same time, due to the use of pessimism, \fname~ensures that the agent does not drift to unknown states where the agent cannot predict accurately using the static dataset.

Theoretically, we obtain bounds on the suboptimality of~\fname~which improve over those in prior work. We further showed that this suboptimality bound cannot be improved upon by \emph{any} offline RL algorithm in the worst case. Experimentally, we evaluated~\fname~in the standard continuous control benchmarks in OpenAI gym and showed that it achieves state of the art results. The modular structure of~\fname~comprising of model learning, uncertainty estimation, and model-based planning allows the use of a variety of approaches such as multi-step prediction for model learning, abstention for uncertainty estimation, or model-predictive control for action selection. In future work, we hope to explore these directions.

\section*{Acknowledgements}
The authors thank Prof. Emo Todorov for generously providing the MuJoCo simulator for use in this paper. Rahul Kidambi thanks Mohammad Ghavamzadeh and Rasool Fakoor for pointers to related works and other valuable discussions/pointers about offline RL. Aravind Rajeswaran thanks Profs.~Sham Kakade and Emo Todorov for valuable discussions. The authors also thank Prof. Nan Jiang and Anirudh Vemula for pointers to related work. Rahul Kidambi acknowledges funding from NSF Award $\text{CCF}-1740822$ and computing resources from the Cornell ``Graphite'' cluster. Part of this work was completed when Aravind held dual affiliations with the University of Washington and Google Brain. Aravind acknowledges financial support through the JP Morgan PhD Fellowship in AI. Thorsten Joachims acknowledges funding from NSF Award $\text{IIS}-1901168$. All content represents the opinion of the authors, which is not necessarily shared or endorsed by their respective employers and/or sponsors.

\section*{Broader Impact}

This paper studies offline RL, which allows for data driven policy learning using pre-collected datasets. The ability to train policies offline can expand the range of applications where RL can be applied as well as the sample efficiency of any downstream online learning. Since the dataset has already been collected, offline RL enables us to abstract away the exploration or data collection challenge. Safe exploration is crucial for applications like robotics and healthcare, where poorly designed exploratory actions can have harmful physical consequences. Avoiding online exploration by an autonomous agent, and working with a safely collected dataset, can have the broader impact of alleviating safety challenges in RL. That said, the impact of RL agents to the society at large is highly dependent on the design of the reward function. If the reward function is designed by malicious actors, any RL agent, be it offline or not, can present negative consequences. Therefore, the design of reward functions requires checks, vetting, and scrutiny to ensure RL algorithms are aligned with societal norms.

\bibliography{references}
\bibliographystyle{unsrtnat}

\clearpage
\appendix

\section{Theoretical Results: Proofs For Section~\ref{ssec:theory}}\label{app:proofs}
In this section, we present the proofs of our main results Theorem~\ref{thm:pess-world} and Proposition~\ref{prop:counter}.
\begin{proof}[Proof of Theorem~\ref{thm:pess-world}]
We wish to show the following two inequalities.
\begin{align*}
    J_{\hat{\rho_0}}(\pi, \Mp) &\geq J_{\rho_0}(\pi, \M) - \frac{2\Rmax}{1-\gamma} \cdot D_{TV}(\rho_0, \hat{\rho_0}) - \frac{2 \gamma  \Rmax}{(1-\gamma)^2} \cdot \alpha - \frac{2 \Rmax}{1-\gamma} \cdot \E{\gamma^{T_{\U}^\pi}}, \mbox{ and } \\
    J_{\hat{\rho_0}}(\pi, \Mp) &\leq J_{\rho_0}(\pi, \M) + \frac{2\Rmax}{1-\gamma} \cdot D_{TV}(\rho_0, \hat{\rho_0}) + \frac{2 \gamma \Rmax}{(1-\gamma)^2}\cdot \alpha.
\end{align*}

The proof of this theorem is inspired by the simulation lemma of~\cite{KearnsSingh2002}, with some additional modifications due to pessimism, and goes through the pessimistic MDP $\Mpe$, which is the same as $\Mp$ except that the starting state distribution is $\rho_0$ instead of $\hat{\rho}_0$ and the transition probability from a known state-action pair $(s,a)$ is $P(s'|s,a)$ instead of $\widehat{P}(s'|s,a)$. More concretely, $\Mpe$ is described by $\{S\cup\textrm{HALT}, A, r_p, {P}_p, {\rho}_0, \gamma\}$, where HALT is an additional absorbing state we introduce similar to what we did for $\Mp$. The modified reward and transition dynamics are given by:
\begin{equation*}
\begin{split}
    {P}_{p}(s'|s,a) = 
    \begin{cases}
    \delta(s'=\mathrm{HALT}) & 
    \begin{array}{l}
            \mathrm{if} \  U^\alpha(s,a) = \mathrm{TRUE}   \\
          \mathrm{or} \ \ s = \mathrm{HALT}
    \end{array} \\
    {P}(s'|s,a) & \mbox{\ \ otherwise}.
    \end{cases}
\end{split}
\quad
\begin{split}
    r_p(s,a) = 
    \begin{cases}
    -\nrew & \mathrm{if} \  s = \mathrm{HALT} \\
    r(s,a) & \mbox{otherwise}
    \end{cases}
\end{split}
\end{equation*}

We first show that
\begin{align*}
    J_{\hat{\rho_0}}(\pi,\Mp) &\geq J_{{\rho_0}}(\pi,\Mpe) - \frac{2\Rmax}{1-\gamma} \cdot D_{TV}(\rho_0, \hat{\rho_0}) - \frac{2 \gamma  \Rmax}{(1-\gamma)^2} \cdot \alpha, \mbox{ and } \\
    J_{\hat{\rho_0}}(\pi,\Mp) &\leq J_{\rho_0}(\pi,\Mpe) + \frac{2\Rmax}{1-\gamma} \cdot D_{TV}(\rho_0, \hat{\rho_0}) + \frac{2 \gamma  \Rmax}{(1-\gamma)^2} \cdot \alpha,
\end{align*}
The main idea is to couple the evolutions of any given policy on the pessimistic MDP $\Mpe$ and the model-based pessimistic MDP $\Mp$ so that $(s_{t-1},a_{t-1}) \defeq (s_{t-1}^{\Mpe}, a_{t-1}^{\Mpe}) = (s_{t-1}^{\Mp}, a_{t-1}^{\Mp})$. 

Assuming that such a coupling can be performed in the first step, since $\norm{P(s,a) - \hat{P}(s,a)}_1 \leq \alpha$, this coupling can be performed at each subsequent step with probability $1-\alpha$. The probability that the coupling is not valid at time $t$ is at most $1 - (1-\alpha)^t$. So the total difference in the values of the policy $\pi$ on the two MDPs can be upper bounded as:
\begin{align*}
    \abs{J_{\hat{\rho_0}}(\pi,\Mp) - J_{\rho_0}(\pi,\Mpe)} &\leq \frac{2 R_{\max}}{1-\gamma} \cdot D_{TV}(\rho_0, \hat{\rho_0}) + \sum_t \gamma^t \left( 1 - (1-\alpha)^t \right) \cdot 2\cdot \Rmax \\
    &\leq \frac{2\Rmax}{1-\gamma} \cdot D_{TV}(\rho_0, \hat{\rho_0}) + \frac{2 \gamma  \Rmax}{(1-\gamma)^2} \cdot \alpha.
\end{align*}

We now argue that
\begin{align*}
    J_{\rho_0}(\pi,\Mpe) &\geq J_{\rho_0}(\pi,\M) - \frac{2 \Rmax}{1-\gamma} \cdot \E{\gamma^{T_u^\pi}}, \mbox{ and } \\
    J_{\rho_0}(\pi,\Mpe) &\leq J_{\rho_0}(\pi,\M).
\end{align*}
For the first part, we see that the evolution of any policy $\pi$ on the pessimistic MDP $\Mpe$, can be coupled with the evolution of $\pi$ on the actual MDP $\M$ until $\pi$ encounters an unknown state. From this point, the total rewards obtained on the pessimistic MDP $\Mpe$ will be $\frac{-R_{\max}}{1-\gamma}$, while the maximum total reward obtained by $\pi$ on $\M$ from that point on is $\frac{\Rmax}{1-\gamma}$. Multiplying by the discount factor $\E{\gamma^{T_u^\pi}}$ proves the first part.

For the second part, consider any policy $\pi$ and let it evolve on the MDP $\M$ as $\left(s,a,s'_{\M}\right)$. Simulate an evolution of the same policy $\pi$ on $\Mpe$, $\left(s,a,s'_{\Mpe}\right)$, as follows: if $(s,a) \in SA_k$, then $s'_{\Mpe} = s'_{\M}$ and if $(s,a) \in \U$, then $s'_{\Mpe} = \textrm{HALT}$. We see that the rewards obtained by $\pi$ on each transition in $\Mpe$ is less than or equal to that obtained by $\pi$ on the same transition in $\M$. This proves the second part of the lemma.
\end{proof}


\begin{lemma}\label{lem:disc}(Hitting time and visitation distribution)
For any set $\mathcal{S} \subseteq S \times A$, and any policy $\pi$, we have
$\E{\gamma^{T_{\mathcal{S}}^{\pi}}} \leq \frac{1}{1-\gamma} \cdot d^{\pi,\M}(\mathcal{S})$.
\end{lemma}
\begin{proof}[Proof of Lemma~\ref{lem:disc}]
The proof is rather straightforward. We have
\begin{align*}
    \E{\gamma^{T_{\U}^{\pi}}} &\leq \sum_{(s',a')\in \U} \E{\gamma^{T_{(s',a')}^{\pi}}} \leq \sum_{(s',a')\in \U}\sum_{t=0}^{\infty}{\gamma^t  P(s_t=s',a_t=a'|s_0\sim \rho_0,\pi,\M)} \\
    &= \frac{1}{1-\gamma} \sum_{(s',a')\in \U} d^{\pi, \M}(s',a') = \frac{1}{1-\gamma} \cdot d^{\pi, \M}(\U).
\end{align*}
\end{proof}

\begin{figure}[t]
    \centering
    \includegraphics[scale=0.5]{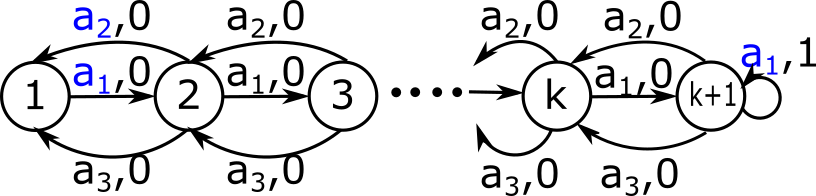}
    \caption{This example shows that the suboptimality of any offline RL algorithm is at least $\frac{\Rmax}{4(1-\gamma)^2} \times \frac{d^{\pi^*,\M}(\U_D)}{\log \frac{1}{1-\gamma}}$ in the worst case and hence Corollary~\ref{cor:subopt} is tight. The states $1, 2, \cdots, k+1$ in the MDP are depicted under the circles. The actions $a_1,a_2,a_3$, rewards and transitions are depicted on the arrows connecting the states. The actions taken by the behavior (i.e. the data collection) policy are depicted in {\color{blue} blue}. See Proposition~\ref{prop:counter} and its proof for more details.}
    \label{fig:counterexample}
\end{figure}

\begin{proof}[Proof of Proposition~\ref{prop:counter}]
We consider the MDP in Figure~\ref{fig:counterexample}, where we set $k=10 \log \frac{1}{1-\gamma}$.
The MDP has $k+1$ states, with three actions $a_1, a_2$ and $a_3$ at each state. The rewards (shown on the transition arrows) are all $0$ except for the action $a_1$ taken in state $k+1$, in which case it is $1$. Note that the rewards can be scaled by $\Rmax$ but for simplicity, we consider the setting with $\Rmax = 1$. It is clear that the optimal policy $\pi^*$ is to take the action $a_1$ in all the states. The starting state distribution $\rho_0$ is state $1$ with probability $p_0 \defeq \frac{\epsilon}{(1-\gamma)\log \frac{1}{1-\gamma}}$ and state $k+1$ with probability $1-p_0$. The actions taken by the data collection policy are shown in {\color{blue} blue}. Since the dataset consists only of (state, action, reward, next state) pairs $(1,a_1,0,2), (2,a_2, 0, 1)$ and $(k+1,a_1,1,k+1)$ we see that $\U_D = (S \times A) \setminus \set{(1,a_1),(2,a_2), (k+1,a_1)}$ and $d^{\pi^*,\M}(\U_D) = (1-\gamma) \cdot \sum_{t=1}^{k-1} \gamma^t \cdot p_0 \leq (1-\gamma) \cdot (k-1) \cdot p_0 \leq \epsilon$ proving the first claim.
Since none of the states and actions in $\U_D$ are seen in the dataset, after permuting the actions if necessary, the expected time taken by any policy learned from the dataset, to reach state $k+1$ starting from state $1$ is at least $\exp\left(k/5\right) \geq (1-\gamma)^{-2}$. So, the value of any policy $\hat{\pi}$ learned from the dataset is at most $\frac{1-p_0}{1-\gamma} + \frac{p_0 \cdot \gamma^{(1-\gamma)^{-2}}}{1-\gamma} = \frac{1}{1-\gamma} - p_0 \cdot \frac{1-\gamma^{(1-\gamma)^{-2}}}{1-\gamma} \leq \frac{1}{1-\gamma} - \frac{3p_0}{4(1-\gamma)}$, where we used $\gamma \in [0.95,1)$ in the last step.
On the other hand, the value of $\pi^*$ is at least $\frac{1-p_0}{1-\gamma} + p_0 \cdot \left(\frac{1}{1-\gamma} - k\right) $. So the suboptimality of any learned policy is at least $p_0 \cdot \left( \frac{3}{4(1-\gamma)} - k\right) = p_0 \cdot \left( \frac{3}{4(1-\gamma)} - 10 \log \frac{1}{1-\gamma}\right) \geq \frac{p_0}{4 (1-\gamma)}$, where we again used $\gamma \in [0.95,1)$ in the last step. Substituting the value of $p_0$ proves the proposition.
\end{proof}
\begin{proof}[Proof of Lemma~\ref{lem:asymptotic}]
We first note that the empirical starting distribution $\hat{\rho_0}$ satisfies $D_{TV}(\rho_0, \hat{\rho_0}) \leq \frac{C}{\rho_{0,\textrm{min}}} \cdot \sqrt{\frac{\log \frac{1}{\delta\rho_{0,\textrm{min}}}}{n}}$, for a large enough constant $C$. This is because for each state $s$ in the support of $\rho_0$, its empirical frequency in $\data$ satisfies:
\begin{align*}
    \abs{\hat{\rho_0}(s) - \rho_0(s)} \leq C \sqrt{\frac{\log \frac{1}{\delta\rho_{0,\textrm{min}}}}{n}},
\end{align*}
with probability at least $1-\delta\rho_{0,\textrm{min}}$ using Chernoff's bound, where $C$ is an absolute numerical constant. Using union bound over at most $\frac{1}{\rho_{0,\textrm{min}}}$ states in the support of $\rho_0$, we see that with probability at least $1-\delta$, we have $D_{TV}(\rho_0, \hat{\rho_0}) \leq \frac{C}{\rho_{0,\textrm{min}}} \cdot \sqrt{\frac{\log \frac{1}{\delta\rho_{0,\textrm{min}}}}{n}}$.

Similarly, for any state action pair $(s,a)$, denoting $n_{(s,a)}$ as the number of times $(s,a)$ appears in $\data$, we have that:
\begin{align*}
    \frac{n_{(s,a)}}{n} - d^{\pi_b,\M}(s,a) \geq -C\sqrt{\frac{\log \frac{1}{\delta d^{\pi_b}_\textrm{min}}}{n}},
\end{align*}
with probability at least $1- {d^{\pi_b}_\textrm{min}\delta}{}$. Again using a union bound over all state-action pairs in the support of $d^{\pi_b,\M}(\cdot)$, we see that:
\begin{align*}
    n_{(s,a)} \geq d^{\pi_b}_\textrm{min} \cdot n - C\sqrt{n \cdot {\log \frac{1}{\delta d^{\pi_b}_\textrm{min}}}{}},
\end{align*}
for every $(s,a)$ in the support of $d^{\pi_b,\M}(\cdot)$ with probability at least $1-\delta$. The asssumption on the size of $n$ then implies that $n_{(s,a)} \geq \frac{d^{\pi_b}_\textrm{min} \cdot n}{2}$. Using a similar Chernoff bound argument, we see that $D_{TV}(P(\cdot|s,a), \hat{P}(\cdot|s,a)) \leq \frac{C}{p_\textrm{min}} \cdot \sqrt{\frac{\log \frac{1}{\delta p_\textrm{min} d^{\pi_b}_\textrm{min}}}{n_{(s,a)}}}$ for every $(s,a)$ in the support of $d^{\pi_b,\M}(\cdot)$ with probability at least $1-\delta$. By choosing $\alpha = \frac{C}{p_\textrm{min}} \cdot \sqrt{\frac{\log \frac{1}{\delta p_\textrm{min} d^{\pi_b}_\textrm{min}}}{n_{(s,a)}}}$, we see that $\U \cap \textrm{Supp}(d^{\pi_b,\M}) = \emptyset$ and hence $T_\U^{\pi_b}=\infty$.
By Theorem~\ref{thm:pess-world}, we have that for any policy $\pi$, we have:
\begin{align*}
    J_{\rho_0}(\pi_{\textrm{out}}, \M) &\geq J_{\hat{\rho_0}}(\pi_{\textrm{out}}, \Mp) - \frac{2 \Rmax}{1-\gamma} \cdot D_{TV}(\rho_0, \hat{\rho_0}) - \frac{2\gamma \Rmax}{(1-\gamma)^2} \cdot \alpha \\
    &\geq J_{\hat{\rho_0}}(\pi, \Mp) - \epsilon_{\pi} - \frac{2 \Rmax}{1-\gamma} \cdot D_{TV}(\rho_0, \hat{\rho_0}) - \frac{2\gamma \Rmax}{(1-\gamma)^2} \cdot \alpha \\
    &\geq J_{{\rho_0}}(\pi, \M) - \epsilon_{\pi} - \frac{4 \Rmax}{1-\gamma} \cdot D_{TV}(\rho_0, \hat{\rho_0}) - \frac{4\gamma \Rmax}{(1-\gamma)^2} \cdot \alpha - \frac{2 \Rmax}{1-\gamma} \cdot \E{\gamma^{T_{\U}^{\pi}}}.
\end{align*}
Plugging $\pi = \pi_b$ gives us the first assertion and plugging $\pi = \pi^*$ and using Lemma~\ref{lem:disc} gives us the second assertion.
\end{proof}
\vspace*{-10pt}
\section{Detailed Related Work}\label{app:related}
Our work takes a model-based approach to offline RL. We review related work pertaining to both of these domains in this section.
\vspace*{-5pt}
\subsection{Offline RL}
Offline RL dates at least to the work of Lange et al.~\citep{LangeGR12}. In this setting, an RL agent is provided access to a typically large offline dataset, using which it has to produce a highly rewarding policy. This has direct applications in fields like healthcare~\citep{Gottesman18,WangZHZ18,YuR019}, recommendation systems~\citep{StrehlLK10, SwaminathanJ15, CovingtonAS16, ChenBCJBC18}, dialogue systems~\citep{ZhouSRE17,JaquesGHSFLJGP19,Karampatziakis19}, and autonomous driving~\citep{SallabAPY17}. We refer the readers to the review paper of Levine et al.~\citep{LevineKTF20} for an overview of potential applications. On the algorithmic front, prior work in offline RL can be broadly categorized into three groups as described below.

\paragraph{Importance sampling}
The first approach to offline RL is through importance sampling. In this approach, trajectories from the offline dataset are directly used to estimate the policy gradient, which is subsequently corrected using importance weights. This approach is particularly common in contextual bandits literature~\citep{LiCLW10,StrehlLK10,SwaminathanJ15} where the importance weights are relatively easier to estimate due to the non-sequential nature of the problem. For MDPs, Liu et al.~\citep{LiuSAB19} present an importance sampling based off-policy policy gradient method by estimating state distribution weights~\citep{HallakM17,GeladaB19,NachumDualDICE}. The work of Liu et al.~\citep{LiuSAB19} also utilizes the notion of {\em pessimism} by optimizing only over a subset of states visited by the behavioral policy. They utilize importance weighted policy gradient (with estimated importance weights) to optimize this MDP. However, their work does not naturally capture a notion of generalization over the state space. Moreover, their results require strong assumptions on the data collecting policy in the sense of ensuring support on states visited by the optimal policy. Our framework,~\fname,~provides the same guarantees under identical assumptions, but we also show that the performance of~\fname~degrades gracefully when these assumptions aren't satisfied.

\paragraph{Dynamic programming} The overwhelming majority of recent algorithmic work in offline RL is through the paradigm of approximate dynamic programming. In principle, any off-policy algorithm based on Q-learning~\cite{Watkins89, MnihKSRVBGRFOPB15} or actor-critic architectures~\cite{LillicrapHPHETS15, FujimotoTD3, HaarnojaSAC} can be used with a static offline dataset. However, recent empirical studies confirm that such a direct extension leads to poor results due to the challenges of overestimation bias in generalization and distribution shift. To address overestimation bias, prior work has proposed approaches like ensembles of Q-networks~\cite{AgarwalSN19, FujimotoBCQ, JaquesGHSFLJGP19}. As for distribution shift, the principle approach used is to regularize the learned policy towards the data logging policy~\cite{FujimotoBCQ, KumarBEAR, WuTN19BRAC}. Different regularization schemes, such as those based on KL-divergence and maximum mean discrepancy (MMD), have been considered in the past. Wu et al.~\cite{WuTN19BRAC} perform a comparative study of such regularization schemes and find that they all perform comparably. ADP-based offline RL has also be studied theoretically~\cite{Munos2008FiniteTimeBF, Chen2019InformationTheoreticCI}, with \citet{Chen2019InformationTheoreticCI} providing an information-theoretic lower bound on sample complexity.
However, these works again don't study the impact of support mismatch between logging policy and optimal policy. Finally, a recent line of work~\citep{Nachum19AlgaeDICE,NachumD20} focuses on obtaining provably convergent methods for minimizing the (one-step) Bellman error using Duality theory. While they show promising results in continuous control tasks in the online RL setting, their performance in the offline RL setting is yet to be studied.

\paragraph{Model-based RL} The interplay between model-based methods and offline RL has only been sparsely explored. The work of Ross \& Bagnell~\cite{RossB12} theoretically studied the performance of MBRL in the batch setting. In particular, the algorithm they analyzed involves learning a dynamics model using the offline dataset, and subsequently planning in the learned model without any additional safeguards. Their theoretical results are largely negative for this algorithm, suggesting that in the worst case, this algorithm could have arbitrarily large sub-optimality. In addition, their sub-optimality bounds become pathologically loose when the data logging distribution does not share support with the distribution of the optimal policy. Model-based offline RL methods from a safe policy improvement perspective have also been considered~\citep{GhavamzadehPC16}. In contrast to both these works, we present a novel algorithmic framework that constructs and pessimistic MDP, and show that this is crucial for better empirical results and sharper theoretical analysis.

\subsection{Advances in Model-Based RL}

Since our work utilizes model-based RL, we review the most directly related work in the online RL setting. Classical works in MBRL have focused extensively on tabular MDPs and linear quadratic regulartor~(LQR). For tabular MDPs (in the online RL setting), the first known polynomial time algorithms were the model-based algorithms of $E^3$~\cite{KearnsSingh2002} and R-MAX~\cite{Brafman2001RMAXA}. More recent work suggests that model-based methods are minimax optimal for tabular MDPs when equipped with a wide restart state distribution~\cite{AgarwalKY19}. However, these works critically rely on the tabular nature of the problem. Since each table entry is typically considered to be independent, and updates to any entry to do not affect other entries, tabular MDPs do not afford any notion of generalization. The metric-$E^3$~\citep{KakadeKL03} algorithm aims to overcome this challenge by considering an underlying metric space for state-actions that enables generalization. While this work provides a strong theoretical basis, it does not directly provide a practical algorithm that can be used with function approximation. Our work is perhaps conceptually closest to $E^3$ and metric-$E^3$ which partitions the state space into known and unknown regions. A cornerstone of~\fname~is the P-MDP which partitions the state space into known and unknown regions, as in, $E^3$~\cite{KearnsSingh2002} and R-MAX~\cite{Brafman2001RMAXA}, but these constructions were not developed to encourage pessimism. However, all of these works primarily deal with the standard (online) RL setting. Our work differs in its focus on offline RL, where we show the P-MDP construction plays a crucial role. Moreover, direct practical instantiations of $E^3$ and metric-$E^3$ with function approximation have remained elusive.

In recent years, along with an explosion of interest in deep RL, MBRL has emerged as a powerful class of approaches for sample efficient learning. Modern MBRL methods (typically in the online RL setting) can support the use of flexible function approximators like neural networks, as well as generic priors like smoothness and approximate knowledge of physics~\cite{Zeng2019TossingBotLT}, enabling the learning of accurate models. Furthermore, MBRL can draw upon the rich literature on model-based planning including model predictive control (MPC)~\cite{Todorov2005, tassa2012synthesis, WilliamsMPPI, POLO}, search based planning~\cite{MCTSSurvey, RRT}, dynamic programming~\cite{Munos2008FiniteTimeBF, BertsekasBook}, and policy optimization~\cite{Kakade01, SchulmanTRPO, Rajeswaran17nips, SchulmanPPO, HaarnojaSAC}. These advances in MBRL have enabled highly sample efficient learning in widely studied benchmark tasks~\cite{PETS, JannerFZL19, Wang2020ExploringMP, Wang2019BenchmarkingMBRL, RajeswaranGameMBRL}, as well as in a number of challenging robotic control tasks like aggressive driving~\cite{WilliamsMPPI}, dexterous hand manipulation~\cite{PDDM, RajeswaranGameMBRL}, and quadrupedal locomotion~\cite{Yang2019DataER}. Among these works, the recent work of Rajeswaran et al.~\cite{RajeswaranGameMBRL} demonstrated state of the art results with MBRL in a range of benchmark tasks, and forms the basis for our practical implementation. In particular, our model learning and policy optimization subroutines are extended from the MAL framework in Rajeswaran et al.~\cite{RajeswaranGameMBRL}. However, our work crucially differs from it due to the pessimistic MDP construction, which we show is important for success in the offline RL setting.

\section{Additional Experimental Details And Setup}\label{app:expts}

\subsection{Environment Details And Setup} 
As mentioned before, following recent efforts in offline RL~\citep{FujimotoBCQ, KumarBEAR, WuTN19BRAC}, we consider four continuous control tasks:~\hopper,~\hc,~\ant,~\walker~from OpenAI gym~\citep{Brockman16} simulated with MuJoCo~\cite{TodorovET12}. As normally done in MBRL literature with OpenAI gym tasks~\citep{KurutachCDTA18, NagabandiMBMF, LuoXLTDM19, RajeswaranGameMBRL}, we reduce the planning horizon for the environments to 400 or 500. Similar to~\citep{LuoXLTDM19, RajeswaranGameMBRL}, we append our state parameterization with center of mass velocity to compute the reward from observations. Mirroring realistic settings, we assume access to data collected using a partially trained (sub-optimal) policy interacting with the environment. To obtain a partially trained policy $\pi_{p}$~\citep{FujimotoBCQ, KumarBEAR, WuTN19BRAC}, we run (online) TRPO~\citep{SchulmanTRPO} until the policy reaches a value of $1000$, $4000$, $1000$, $1000$ respectively for these environments. This policy in conjunction with exploration strategies are used to collect the datasets (see below for more details). All our results are obtained by averaging runs of five random seeds (for the planning algorithm), with the seed values being $123,246,369,492,615$. Each of our experiments are run with 1 NVidia GPU and 2 CPUs using a total of 16GB of memory.

\subsection{Dynamics Model, Policy Network And Evaluation}
We use 2 hidden layer MLPs with 512 (for~\hopper,~\walker,~\ant) or 1024 (for~\hc) ReLU activated nodes each for representing the dynamics model, use an ensemble of four such models for building the USAD, and our policy is represented with a 2 hidden layer MLP with 32 tanh activated nodes in each layer. The dynamics model is learnt using Adam~\citep{KingmaB14} and the policy parameters are learnt using model-based NPG steps~\citep{RajeswaranGameMBRL}. We set hyper-parameters and track policy learning curve by performing rollouts in the real environment; these rollouts aren't used for other purposes in the learning procedure. Similar protocols are used in prior work\citep{FujimotoBCQ, KumarBEAR, WuTN19BRAC}.

\subsection{Description Of Types Of Policies}
We build off the experimental setup of~\citep{WuTN19BRAC}. Towards this, we first go over some notation. Firstly, let $\pi_b$ represent the behavior policy, $\pi_r$ is a random policy that picks actions according to a certain probability distribution (for e.g., Gaussian $\pi_r^g$/Uniform $\pi_r^u$ etc.), $\pi_p$ a partially-trained policy, which one can assume is better than a random policy in value. Let $\pi_b^u(q)$ be a policy that plays random actions with probability $q$, and sampled actions from $\pi_b$ with probability $1-q$. Let $\pi_b^g(\beta)$ be a policy that adds zero-mean Gaussian noise with standard deviation $\beta$ to actions sampled from $\pi_b$.
Consider a behavior policy which, for instance, can be a partially trained data logging policy $\pi_b$. We consider five different exploration strategies, each corresponding to adding different kinds of exploratory noise to $\pi_b$, as described below.

\subsection{Datasets And Exploration Strategies}
For each environment, we use a combination of a behavior policy $\pi_b$, a noisy behavior policy $\tilde{\pi}_b$ (see below), and a pure random stochastic process $\pi_r$ to collect several datasets, following Wu et al.~\cite{WuTN19BRAC}. Each dataset contains the equivalent of 1 million timesteps of interactions with the environment. See below for detailed instructions.
\begin{enumerate}[leftmargin=*,label=$\mathbf{(\mathcal E\arabic*)}$]
\item \label{exp:pure} \texttt{Pure}: The entire dataset is collected with the data logging (behavioral) policy $\pi_b$.
\item \label{exp:e1} \texttt{Eps-1}: $40\%$ of the dataset is collected with $\pi_b$, another $40\%$ collected with $\pi_b^u(0.1)$, and the final $20\%$ is collected with a random policy $\pi_r$. 
\item \label{exp:e3} \texttt{Eps-3}: $40\%$ of the dataset is collected with $\pi_b$, another $40\%$ collected with $\pi_b^u(0.3)$, and the final $20\%$ is collected with a random policy $\pi_r$. 
\item \label{exp:g1} \texttt{Gauss-1}: $40\%$ of the dataset is collected with $\pi_b$, another $40\%$ collected with $\pi_b^g(0.1)$, and the final $20\%$ is collected with a random policy $\pi_r$. 
\item \label{exp:g3} \texttt{Gauss-3}: $40\%$ of the dataset is collected with $\pi_b$, another $40\%$ collected with $\pi_b^g(0.3)$, and the final $20\%$ is collected with a random policy $\pi_r$. 
\end{enumerate}

\subsection{Hyperparameter Selection}
Refer to table~\ref{table:experiment-hyperparameters} for details with regards to parameters of~\fname. For all environments and data collection strategies, we learn two-layer MLP based dynamics models with ReLU activations by minimizing the one-step prediction errors using Adam~\citep{KingmaB14} and utilize four of these models for defining the USAD. The negative reward for defining the absorbing unknown state is set as the minimum reward in the dataset $\data$ offsetted by a value that is searched over $\{30,50,100,200\}$.

{\bf Ascertaining unknown state-action pairs:} In order to ascertain unknown state-action pairs, we compute the model disagreement as: $\text{disc}(s,a)=\max_{i\ne j}||f_{\phi_i}(s,a)-f_{\phi_j}(s,a)||_2$, where, $f_{\phi_i}$ and $f_{\phi_j}$ are members of the ensemble of learnt dynamics model. Specifically, we compute $\text{disc}(s,a)$ over all state-action pairs that occur in the static dataset $\data$. Next, we can compute the mean $\mu_d$, standard deviation $\sigma_d$ and the max $m_d$ of the disagreements evaluated for every state-action pair occuring in the dataset. Then, we utilize an upper-confidence inspired strategy by defining a threshold $\text{thresh}= \mu_d + \beta\cdot\sigma_d$. The value of beta is tuned between $0$ to $\beta_{\max} = (m_d-\mu_d)/\sigma_d$ in steps of $5$. For any model-based rollout encountered during planning, if the discrepancy of the state-action pair at a given timestep exceeds $\text{thresh}$, the rollout is truncated at this timestep and is assigned a large negative reward. We emphasize that for every environment, all hyper-parameters (except for $\beta$) is maintained at the same value across all exploration settings.

\begin{table}[htbp]
\vspace*{-4mm}\caption{Hyper-parameters for each environment for~\fname. Note that most hyper-parameters are common across domains, and the differences are primarily in reward penalty and number of fitting epochs, which are necessarily environment specific.}\vspace{10pt}
\begin{minipage}{.49\linewidth}
\begin{adjustbox}{width=\columnwidth,center}
\begin{tabular}{l|c}
\toprule
\multicolumn{2}{c}{\bf Environment: Ant-v2}\\
\midrule
		Parameter & Value \\
		\midrule
		Dynamics Model & MLP(512, 512)\\
		Activation & ReLU\\
		\# Training Epochs & 300\\
		Adam Stepsize & 5e-4\\
		Batch Size & 256\\
		Horizon & 500\\
		{Negative Reward} & $r_{\min}(\data)-100$\\
		USAD & 4-dynamics models\\
\bottomrule
\end{tabular}
\end{adjustbox}
\end{minipage}
\hspace*{5pt}
\begin{minipage}{.49\linewidth}
\begin{adjustbox}{width=\columnwidth,center}
\begin{tabular}{l|c}
\toprule
\multicolumn{2}{c}{\bf Environment: Hopper-v2}\\ %
\midrule
		Parameter & Value \\
		\midrule
		Dynamics Model & MLP(512, 512)\\
		Activation & ReLU\\
		\# Training Epochs & 300\\
		Adam Stepsize & 5e-4\\
		Batch Size & 256\\
		Horizon & 400\\
		Negative Reward & $r_{\min}(\data)-50$\\
		USAD & 4-dynamics models\\
\bottomrule
\end{tabular}
\end{adjustbox}
\end{minipage}
\vspace{10pt}
\begin{minipage}{.49\linewidth}
\begin{adjustbox}{width=\columnwidth,center}
\begin{tabular}{l|c}
\toprule
\multicolumn{2}{c}{\bf Environment: HalfCheetah-v2}\\
\midrule
		Parameter & Value \\
		\midrule
		Dynamics Model & MLP(1024,1024)\\
		Activation & ReLU\\
		\# Training Epochs & 3000\\
		Adam Stepsize & 5e-4\\
		Batch Size & 256\\
		Horizon & 500\\
		Negative Reward & $r_{\min}(\data)-200$\\
		USAD & 4-dynamics models\\
\bottomrule
\end{tabular}
\end{adjustbox}
\end{minipage}
\hspace*{5pt}
\begin{minipage}{.49\linewidth}
\begin{adjustbox}{width=\columnwidth,center}
\begin{tabular}{l|c}
\toprule
\multicolumn{2}{c}{\bf Environment: Walker-v2}\\ %
\midrule
		Parameter & Value \\
		\midrule
		Dynamics Model & MLP(512, 512)\\
		Activation & ReLU\\
		\# Training Epochs & 300\\
		Adam Stepsize & 5e-4\\
		Batch Size & 256\\
		Horizon & 400\\
		Negative Reward & $r_{\min}(\data)-30$\\
		USAD & 4-dynamics models\\
\bottomrule
\end{tabular}
\end{adjustbox}
\end{minipage}%
\label{table:experiment-hyperparameters}
\end{table}

\begin{table}[htbp]
\vspace*{-4mm}\caption{Hyper-parameters for model-based policy optimization. Note that most hyperparameters are common except the number of iterations and exploration noise.}\vspace{10pt}
\begin{minipage}{.49\linewidth}
\begin{adjustbox}{width=\columnwidth,center}
\begin{tabular}{l|c}
\toprule
\multicolumn{2}{c}{\bf Environment: Ant-v2}\\
\midrule
		Parameter & Value \\
		\midrule
		Policy Net & MLP(32,32)\\
		Non-linearity & Tanh\\
		\# updates & 1000\\
		$\log\sigma_{\text{init}}$ & -0.1\\
		$\log\sigma_{\text{min}}$ & -2.0\\
		\# trajectories for gradient & 200\\
		\# Eval trajectories & 25\\
		\# CG Steps/Damping & 10, 1e-4\\
\bottomrule
\end{tabular}
\end{adjustbox}
\end{minipage}
\hspace*{5pt}
\begin{minipage}{.49\linewidth}
\begin{adjustbox}{width=\columnwidth,center}
\begin{tabular}{l|c}
\toprule
\multicolumn{2}{c}{\bf Environment: Hopper-v2}\\ %
\midrule
		Parameter & Value \\
		\midrule
		Policy Net & MLP(32,32)\\
		Non-linearity & Tanh\\
		\# updates & 500\\
		$\log\sigma_{\text{init}}$ & -0.25\\
		$\log\sigma_{\text{min}}$ & -2.0\\
		\# trajectories for gradient & 50\\
		\# Eval trajectories & 25\\
		\# CG Steps/Damping & 25, 1e-4\\
\bottomrule
\end{tabular}
\end{adjustbox}
\end{minipage}
\vspace{10pt}
\begin{minipage}{.49\linewidth}
\begin{adjustbox}{width=\columnwidth,center}
\begin{tabular}{l|c}
\toprule
\multicolumn{2}{c}{\bf Environment: HalfCheetah-v2}\\
\midrule
		Parameter & Value \\
		\midrule
		Policy Net & MLP(32,32)\\
		Non-linearity & Tanh\\
		\# updates & 2500\\
		$\log\sigma_{\text{init}}$ & -1.0\\
		$\log\sigma_{\text{min}}$ & -2.0\\
		\# trajectories for gradient & 40\\
		\# Eval trajectories & 25\\
		\# CG Steps/Damping & 25, 1e-4\\
\bottomrule
\end{tabular}
\end{adjustbox}
\end{minipage}
\hspace*{5pt}
\begin{minipage}{.49\linewidth}
\begin{adjustbox}{width=\columnwidth,center}
\begin{tabular}{l|c}
\toprule
\multicolumn{2}{c}{\bf Environment: Walker-v2}\\ %
\midrule
		Parameter & Value \\
		\midrule
		Policy Net & MLP(32,32)\\
		Non-linearity & Tanh\\
		\# updates & 1000\\
		$\log\sigma_{\text{init}}$ & -0.5\\
		$\log\sigma_{\text{min}}$ & -2.0\\
		\# trajectories for gradient & 100\\
		\# Eval trajectories & 25\\
		\# CG Steps/Damping & 25, 1e-4\\
\bottomrule
\end{tabular}
\end{adjustbox}
\end{minipage}\vspace*{-7mm}
\label{table:policy-optimization-hyperparameters}
\end{table}
With regards to the policy and the planning algorithm, we consider a  $(32,32)$ tanh MLP optimized using normalized model-based NPG steps (see, for instance, the work of~\citet{RajeswaranGameMBRL} for the model-based NPG algorithm). Parameters of model-based NPG is described in table~\ref{table:policy-optimization-hyperparameters}.

\clearpage

\begin{table}[H]
\centering
\vspace*{-20pt}
\caption{Results in the four environments and five exploration configurations. }
\label{table:all_results_full}
\begin{tabular}{lccccc}
\toprule
\multicolumn{3}{c}{\bf Environment: Ant-v2} & \multicolumn{3}{c}{\bf Partially trained policy: 1241} \\
\midrule
Algorithm &
\vtop{\hbox{\strut \texttt{Pure}}\hbox{\strut\ref{exp:pure}}} &
\vtop{\hbox{\strut \texttt{Eps-1}}\hbox{\strut\ref{exp:e1}}} &
\vtop{\hbox{\strut \texttt{Eps-3}}\hbox{\strut\ref{exp:e3}}} &
\vtop{\hbox{\strut \texttt{Gauss-1}}\hbox{\strut\ref{exp:g1}}} &
\vtop{\hbox{\strut \texttt{Gauss-3}}\hbox{\strut\ref{exp:g3}}}
\\
\midrule
SAC~\citep{HaarnojaSAC} & 0 & -1109 & -911 & -1071 & -1498\\
BC & 1235 & 1300 & 1278 & 1203 & 1240\\
BCQ~\citep{FujimotoBCQ} & 1921 & 1864 & 1504 & 1731 & 1887 \\
BEAR~\citep{KumarBEAR} & 2100 & 1897 & 2008 & 2054 & 2018 \\
MMD\_vp~\citep{WuTN19BRAC}      & \underline{2839}  & \underline{2672}  & \underline{2602} & \underline{2667} & 2640 \\
KL\_vp~\citep{WuTN19BRAC}       & 2514  & 2530  & 2484  & 2615 & \underline{2661} \\
KL\_dual\_vp~\citep{WuTN19BRAC} &2626 & 2334 & 2256 & 2404 & 2433 \\
W\_vp~\citep{WuTN19BRAC}       & 2646  & 2417  & 2409 & 2474 & 2487 \\
MMD\_pr~\citep{WuTN19BRAC}      & 2583  & 2280  & 2285 & 2477 & 2435 \\
KL\_pr ~\citep{WuTN19BRAC}      & 2241  & 2247  & 2181  & 2263 & 2233 \\
KL\_dual\_pr~\citep{WuTN19BRAC} &2218 & 1984 & 2144 & 2215 & 2201 \\
W\_pr ~\citep{WuTN19BRAC}      & 2241  & 2186  & 2284  & 2365 & 2344 \\
\midrule
Best Baseline & 2839 & 2672 & 2602 & 2667 & 2661\\
\midrule
\fname~(Ours) & \hbox{\strut {\bf 3663 $\pm$247}} & \vtop{\hbox{\strut \bf{3305 $\pm$413}}} & \vtop{\hbox{\strut {\bf 3008 $\pm$231}}} & \vtop{\hbox{\strut {\bf3329 $\pm$270}}} & \vtop{\hbox{\strut {\bf 3693 $\pm$33}}} \\
\bottomrule
\end{tabular}

\begin{tabular}{lccccc}
\toprule
\multicolumn{3}{c}{\bf Environment: Hopper-v2} & \multicolumn{3}{c}{\bf Partially trained policy: 1202} \\
\midrule
Algorithm &
\vtop{\hbox{\strut \texttt{Pure}}\hbox{
    \strut\ref{exp:pure}
}} &
\vtop{\hbox{\strut \texttt{Eps-1}}\hbox{
    \strut\ref{exp:e1}
}} &
\vtop{\hbox{\strut \texttt{Eps-3}}\hbox{
    \strut\ref{exp:e3}
}} &
\vtop{\hbox{\strut \texttt{Gauss-1}}\hbox{
    \strut\ref{exp:g1}
}} &
\vtop{\hbox{\strut \texttt{Gauss-3}}\hbox{\strut
    \ref{exp:g3}}}
\\
\midrule
SAC~\citep{HaarnojaSAC} &0.2655 & 661.7 & 701 & 311.2 & 592.6 \\
BC &1330 & 129.4 & 828.3 & 221.1 & 284.6 \\
BCQ~\citep{FujimotoBCQ} &1543 & 1652 & 1632 & 1599 & 1590 \\
BEAR~\citep{KumarBEAR} &0 & 1620 & 2213 & 1825 & 1720 \\
MMD\_vp~\citep{WuTN19BRAC}      & 2291 & 2282 & 1892 & \underline{2255} & 1458 \\
KL\_vp~\citep{WuTN19BRAC}       & \underline{2774} & \underline{2360} & \underline{2892} & 1851 & 2066 \\
KL\_dual\_vp~\citep{WuTN19BRAC} &1735 & 2121 & 2043 & 1770 & 1872 \\
W\_vp~\citep{WuTN19BRAC}       & 2292  & 2187  & 2178  & 1390 & 1739 \\
MMD\_pr~\citep{WuTN19BRAC}      & 2334 & 1688 & 1725 & 1666 & \underline{2097} \\
KL\_pr~\citep{WuTN19BRAC}       & 2574  & 1925  & 2064  & 1688 & 1947 \\
KL\_dual\_pr~\citep{WuTN19BRAC} &2053 & 1985 & 1719 & 1641 & 1551 \\
W\_pr~\citep{WuTN19BRAC}       & 2080  & 2089  & 2015  & 1635 & \underline{2097} \\
\midrule
Best Baseline & 2774 & 2360 & 2892 & 2255 & 2097\\
\midrule
\fname~(Ours) & \hbox{\strut {\bf 3642 $\pm$54}} & \vtop{\hbox{\strut \bf{3724 $\pm$46}}} & \vtop{\hbox{\strut {\bf 3535 $\pm$91}}} & \vtop{\hbox{\strut {\bf 3653 $\pm$52}}} & \vtop{\hbox{\strut {\bf 3648 $\pm$148}}} \\
\bottomrule
\end{tabular}

\begin{tabular}{lccccc}
\toprule
\multicolumn{3}{c}{\bf Environment: Walker-v2} & \multicolumn{3}{c}{\bf Partially trained policy: 1439} \\
\midrule
Algorithm &
\vtop{\hbox{\strut \texttt{Pure}}\hbox{\strut\ref{exp:pure}}} &
\vtop{\hbox{\strut \texttt{Eps-1}}\hbox{\strut\ref{exp:e1}}} &
\vtop{\hbox{\strut \texttt{Eps-3}}\hbox{\strut\ref{exp:e3}}} &
\vtop{\hbox{\strut \texttt{Gauss-1}}\hbox{\strut\ref{exp:g1}}} &
\vtop{\hbox{\strut \texttt{Gauss-3}}\hbox{\strut\ref{exp:g3}}}
\\
\midrule
SAC~\citep{HaarnojaSAC} &131.7 & 213.5 & 127.1 & 119.3 & 109.3 \\
BC &1334 & 1092 & 1263 & 1199 & 1137 \\
BCQ~\citep{FujimotoBCQ} &2095 & 1921 & 1953 & 2094 & 1734 \\
BEAR~\citep{KumarBEAR} &2646 & 2695 & 2608 & 2539 & 2194 \\
MMD\_vp~\citep{WuTN19BRAC}      & 2694  & 3241  & \underline{\bf 3255} & 2893 & \underline{\bf 3368} \\
KL\_vp~\citep{WuTN19BRAC}       & \underline{2907}  & 3175  & 2942 & \underline{\bf 3193} & 3261 \\
KL\_dual\_vp~\citep{WuTN19BRAC} & 2575 & \underline{\bf 3490} & 3236 & 3103 & 3333 \\
W\_vp~\citep{WuTN19BRAC}       & 2635  & 2863  & 2758  & 2856 & 2862 \\
MMD\_pr~\citep{WuTN19BRAC}      & 2670  & 2957 & 2897 & 2759 & 3004 \\
KL\_pr~\citep{WuTN19BRAC}       & 2744  & 2990  & 2747  & 2837 & 2981 \\
KL\_dual\_pr~\citep{WuTN19BRAC} & 2682 & 3109 & 3080 & 2357 & 3155 \\
W\_pr~\citep{WuTN19BRAC}       & 2667  & 3140  & 2928  & 1804 & 2907 \\
\midrule
Best Baseline & 2907 & 3490 & 3255 & 3193 & 3368\\
\midrule
\fname~(Ours) & \hbox{\strut {\bf 3709 $\pm$159}} & \hbox{\strut {2899 $\pm$588}} & \hbox{\strut {\bf 3186 $\pm$92}} & \hbox{\strut {\bf 4027 $\pm$314}} & \hbox{\strut {\bf 2828 $\pm$589}} \\
\bottomrule
\end{tabular}
\end{table}

\begin{table*}[h!]
\centering
\begin{tabular}{lccccc}
\toprule
\multicolumn{3}{c}{\bf Environment: HalfCheetah-v2} & \multicolumn{3}{c}{\bf Partially trained policy: 4206} \\
\midrule
Algorithm &
\vtop{\hbox{\strut \texttt{Pure}}\hbox{\strut\ref{exp:pure}}} &
\vtop{\hbox{\strut \texttt{Eps-1}}\hbox{\strut\ref{exp:e1}}} &
\vtop{\hbox{\strut \texttt{Eps-3}}\hbox{\strut\ref{exp:e3}}} &
\vtop{\hbox{\strut \texttt{Gauss-1}}\hbox{\strut\ref{exp:g1}}} &
\vtop{\hbox{\strut \texttt{Gauss-3}}\hbox{\strut\ref{exp:g3}}}
\\
\midrule
SAC~\citep{HaarnojaSAC} &5093 & 6174 & 5978 & 6082 & 6090 \\
BC &4465 & 3206 & 3751 & 4084 & 4033 \\
BCQ~\citep{FujimotoBCQ} &5064 & 5693 & 5588 & 5614 & 5837 \\
BEAR~\citep{KumarBEAR} &5325 & 5435 & 5149 & 5394 & 5329 \\
MMD\_vp~\citep{WuTN19BRAC}      & 6207 & \underline{\bf 6307} & 6263 & \underline{\bf 6323} & \underline{\bf 6400} \\
KL\_vp~\citep{WuTN19BRAC}       & 6104 & 6212 & 6104 & 6219 & 6206 \\
KL\_dual\_vp~\citep{WuTN19BRAC} &\underline{\bf 6209} & 6087 & \underline{\bf 6359} & 5972 & 6340 \\
W\_vp~\citep{WuTN19BRAC}       & 5957  & 6014  & 6001  & 5939 & 6025 \\
MMD\_pr~\citep{WuTN19BRAC}      & 5936 & 6242 & 6166 & 6200 & 6294 \\
KL\_pr~\citep{WuTN19BRAC}       & 6032  & 6116  & 6035  & 5969 & 6219 \\
KL\_dual\_pr~\citep{WuTN19BRAC} &5944 & 6183 & 6207 & 5789 & 6050 \\
W\_pr~\citep{WuTN19BRAC}       & 5897 & 5923 & 5970 & 5894 & 6031 \\
\midrule
Best Baseline & 6209 & 6307 & 6263 & 6323 & 6400\\
\midrule
\fname~(Ours) & \hbox{\strut {6028 $\pm$192}} & \hbox{\strut {5861 $\pm$152}} & \hbox{\strut {5869 $\pm$139}} & \hbox{\strut {6026 $\pm$74}} & \hbox{\strut {5892 $\pm$128}} \\
\bottomrule
\end{tabular}
\end{table*}

\subsection{Ablation Study with the \texttt{Pure-partial} dataset}\label{app:ablationfigs}

\begin{figure}[H]
	\includegraphics[width=\textwidth]{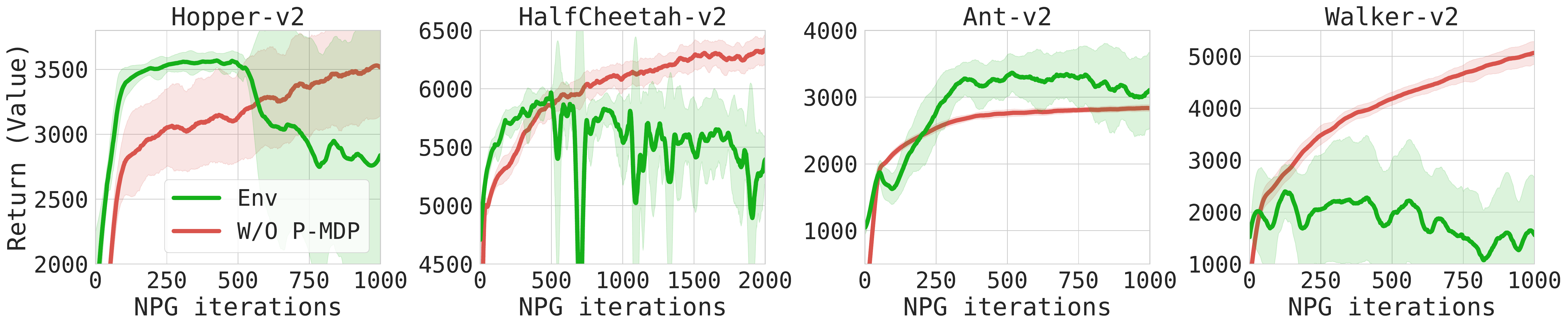}
	\caption{Learning curve using na\"ive MBRL with the \texttt{Pure-partial} dataset. Contrast the learning exhibited by na\"ive MBRL in this figure with~\fname~in Figure~\ref{fig:pessimism_lowerbound}.}
	\label{fig:without_pessimism_purepartial} 
\end{figure}

\subsection{Hyperparameter Guidelines and Ablations}\label{app:hyperparameters}
We did not have resources to perform a thorough hyperparamter search, and largely used our intuitions to guide the choice of hyperparameters. We believe that better results are possible with hyperparameter optimization.  First, we present the influence of the discrepancy threshold for differentiating known and unknown states. We first define the maximum discipancy in the dataset:
\[
\mathrm{disc}_\data = \max_{(s,a) \in \data} \max_{i, j} \| f_i(s,a) - f_j(s,a) \|
\]
where $\data$ denotes offline dataset, and $f_i$ denotes $i^{th}$ dynamics model in the ensemble.

\vspace*{-10pt}

\begin{table}[h!]
\centering
\caption{Influence of discrepancy threshold on the \hopper~task. We use a penalty of $0.0$ along with episode termination for visiting unknown regions in these experiments. We train all the cases for 1000 iterations, and report the average value over the last 100 iterations.}
\label{table:disc_ablation}
\begin{tabular}{@{}ccc@{}}
\toprule
Discrepancy Threshold & Value in P-MDP & Value in true MDP \\ \midrule
$ 0.1 \times \mathrm{disc}_\data$      &    $1315.16$     &   $2082.21$         \\
$ 0.2 \times \mathrm{disc}_\data$      &    $2479.92$     &   $3244.48$         \\
$ 0.5 \times \mathrm{disc}_\data$      &    $3074.75$     &   $3359.66$         \\
$ 1.0 \times \mathrm{disc}_\data$      &    $3543.23$     &   $3595.60$         \\
$ 5.0 \times \mathrm{disc}_\data$      &    $3245.66$     &   $3027.59$         \\
Naive-MBRL                             &    $3656.08$     &   $2809.66$         \\
\bottomrule
\end{tabular}
\end{table}

\newpage 

Our general observations and guidelines for hyperparameters are:
\begin{enumerate}
    \itemsep0em
    \item In Table~\ref{table:disc_ablation}, we first note that $ 0.1 \times \mathrm{disc}_\data$ has the most amount of pessimism and Naive-MBRL has the least/no amount of pessimism. We observe that we obtain best results in the true MDP with an intermediate level of pessimism. Having either too much pessimism or no pessimism both lead to poor results, but for very different reasons that we outline below.
    \item A high degree of pessimism makes policy optimization in the P-MDP difficult. The optimization process may be slow or highly noisy. This is due to non-smoothness introduced in the dynamics and reward due to abrupt changes involving early episode terminations. If difficulty in policy optimization is observed in the P-MDP, we recommend considering reducing the degree of pessimism.
    \item With a lack or low degree of pessimism, policy optimization is typically easier, but the performance in the true MDP might degrade. If it is observed that the value in the P-MDP overestimates the value in the true MDP substantially, then we recommend increasing the degree of pessimism. 
    \item For the tasks considered in this work, positive rewards indicate progress towards the goal. Most of the locomotion tasks involve forward velocity as the primary component of the the reward term. In these cases, we observed that the choice of reward penalty for going into unknown regions did not play a crucial role, as long as it was $\leq 0$. The degree of influence of this parameter in other environments is yet to be determined, and beyond the scope of our empirical study.
\end{enumerate}

\end{document}